\documentclass[preprint,12pt]{elsarticle}
\usepackage{amssymb}
\usepackage{mathrsfs}
\usepackage{mathabx}
\journal{Journal of Logic and Computation}

\usepackage[usenames]{color}
\usepackage{pict2e}
\usepackage[all, knot]{xy}
\usepackage{graphicx}
\xyoption{arc}
\newtheorem{definition}{Definition}

\newtheorem{example}{Example}
\newtheorem{proposition}{Proposition}

\newproof{proof}{Proof}

\begin{document}

\begin{frontmatter}

\title{Formulating Semantics of Probabilistic Argumentation by Characterizing Subgraphs: Theory and Empirical Results\tnoteref{label1}}
%\tnotetext[label1]{The basic idea of this paper has been introduced in \cite{Liao:LORI}.}

\author[cslc,phl,lux]{Beishui Liao}
\ead{baiseliao@zju.edu.cn}

\author[cslc,phl]{Kang Xu%\fnref{label2}
}
\ead{xukanguuu@163.com}

\author[cslc,phl]{Huaxin Huang%\fnref{label2}
}
\ead{rw211@zju.edu.cn}

\address[cslc]{Center for the Study of Language and Cognition, Zhejiang University, China}
\address[phl]{Institute of Logic and Cognition, Zhejiang University, China}
\address[lux]{University of Luxembourg, Luxembourg}

\begin{abstract}
The existing approaches to formulate the semantics of probabilistic argumentation are based on the notion of possible world. Given a probabilistic argument graph (PrAG) with $n$ nodes, up to $2^n$ subgraphs are blindly constructed and their extensions under a given semantics are computed. Then, the probability of a set of arguments $E$ being an extension  under a given semantics $\sigma$ (denoted as $p(E^\sigma)$) is equal to the sum of the probabilities of all subgraphs each of which has the extension $E$. Since many irrelevant subgraphs are constructed, and in many cases, computing  extensions of  subgraphs is computationally intractable, these approaches are fundamentally inefficient or infeasible. In existing literature, while approximate approaches based on {the Monte-Carlo simulation technique} have been proposed to estimate the probability of extensions, how to improve the efficiency of computation {without using the simulation technique} is still an open problem. 
In this paper, we address this problem from the following two perspectives. 
First, conceptually, we define specific properties to characterize the subgraphs of a PrAG with respect to a given extension, such that the probability of a set of arguments $E$ being an extension can be defined in terms of these properties, without (or with less) construction of subgraphs.
Second, computationally, we take preferred semantics as an example, and develop algorithms to evaluate the efficiency of our approach. The results show that  our approach not only dramatically decreases the time for computing $p(E^\sigma)$, but also has an attractive property, which is contrary to that of existing approaches:  the denser the edges of a PrAG are or the bigger the size of a given extension $E$ is, the more efficient our approach computes $p(E^\sigma)$. Meanwhile, it is shown that under complete and preferred semantics, the problems of determining $p(E^\sigma)$ are fixed-parameter tractable. 
\end{abstract}


\end{frontmatter}

\section{Introduction}
In the past two decades, argumentation has been a very active research area in the field of knowledge representation and reasoning, as a nonmonotonic formalism to handle inconsistent and incomplete information by means of constructing, comparing and evaluating arguments. In 1995, Dung proposed a notion of abstract argumentation framework  \cite{Dung:AIJ}, which can be viewed as a directed graph (called \textit{argument graph}, or defeat graph) $G = (A, R)$, in which $A$ is a set of arguments and $R\subseteq A\times A$ is a set of attacks. % (also called nodes and edges of the argument graph respectively). 
Given an argument graph, specific evaluation criteria are defined to determine which arguments can be regarded as justified or acceptable. A set of arguments acceptable together is often called an extension, and the evaluation criteria or sets of extensions of an argument graph {are} called argumentation semantics.
%According to \cite{Dung:AIJ}, extension-based semantics is a formal way to answer this question. 
%Here, an extension represents 
%Under a certain semantics which is defined as a set of evaluation criteria \cite{Baroni:AIJ2007},  arguments that are considered to be acceptable (i.e. able to survive the conflict) together are extensions. 
Dung's abstract argumentation framework and argumentation semantics lay a concrete foundation for the development of various argument systems. 

However, in classical argumentation theory, the uncertainty of arguments and/or attacks is not considered. So, it could be regarded as a purely qualitative formalism. But, in the real world, arguments and/or attacks are often uncertain. So, in recent years, the importance of combining argumentation and uncertainty has been well recognized, and probability-based argumentation is gaining momentum \cite{Dung:comma,Rienstra:AT,Li:TAFA,Dunne:AIJ2011,Hunter:IJAR2013}. In a \textit{probabilistic argument graph} (or  PrAG in brief), each argument is assigned with a probability, denoting the likelihood of the argument appearing in the graph\footnote{A probabilistic argument graph can be defined by assigning probabilities to arguments \cite{Dung:comma,Rienstra:AT,Hunter:comma}, or attacks \cite{Hunter:IJAR2013}, or both arguments and attacks \cite{Li:TAFA}. For simplicity, in this paper we only consider the probabilistic argument graph in which only arguments are associated with probabilities.}. 

Similar to classical argumentation theory, given a PrAG,  a basic problem is to define the status of arguments. The existing approaches are based on {the notion of possible worlds }\cite{Dung:comma,Rienstra:AT,Li:TAFA,Hunter:comma}. Given a PrAG with $n$ nodes, up to $2^n$ subgraphs are blindly constructed. Each subgraph corresponds to a possible world where some arguments appear while other arguments do not appear. The extensions of each subgraph are computed according to classical argumentation semantics.  Then, the probability of a set of arguments $E$ being an extension  under a given semantics $\sigma$ (denoted as $p(E^\sigma)$) is equal to the sum of the probabilities of all subgraphs each of which has the extension $E$. Since many irrelevant subgraphs are constructed, and in many cases, computing  extensions of  subgraphs is computationally intractable, these approaches are fundamentally inefficient or infeasible. In existing literature, while approximate approaches based on {the Monte-Carlo simulation technique} have been proposed to estimate the probability of extensions  \cite{Li:TAFA, Fazzinga:sum13}, how to improve the efficiency of computation without using simulation technique is still an open problem. 

Since the complexity of computing $p(E^\sigma)$ by the existing approaches is mainly caused by blindly constructing subgraphs and computing extensions of each subgraph, an intuitive question arises:

\begin{description}
\item[Intuitive question] Is it possible to compute $p(E^\sigma)$ without (or with less) construction and computation of  subgraphs?
\end{description}

This question has been partially answered by Fazzinga et al  \cite{Fazzinga:ijcai13}. When analyzing the complexity of probabilistic abstract argumentation, they provided a lemma to prove that under admissible and stable semantics, the problem of computing $p(E^\sigma)$ is tractable. In this lemma, $p(E^\sigma)$ is determined by evaluating an expression which only involves the probabilities of the arguments and defeats (attacks) of a probabilistic argument graph\footnote{In \cite{Fazzinga:ijcai13}, probabilities are assinged to both arguments and attacks.}. So, in these cases, no subgraphs are constructed and computed. However, under other semantics (including complete, grounded, preferred and ideal), they only stated that the problem of computing $p(E^\sigma)$  is $FP^{\sharp P}$-complete, without further work on how the above idea can be exploited to improve the efficiency of computation under these semantics.  

Motivated by the intuitive question and the state of the art of computation of probabilistic argumentation, the research problems of the present paper are as follows.

\begin{description}
\item[Research problem 1] Under various argumentation semantics (including not only admissible and stable, but also complete, grounded and preferred, etc.), how to define properties to characterize the subgraphs of a PrAG with respect to an extension $E$, such that $p(E^\sigma)$ can be computed by using these properties, rather than by blindly constructing and computing all subgraphs of the PrAG?
\item[Research problem 2] How to evaluate the efficiency of the new approach?
\end{description}

With these two research problems in mind, the rest of this paper is organized as follows.  In Section 2, some notions of abstract argumentation and probabilistic abstract argumentation are reviewed to make this paper self-contained. In Sections 3 and 4, to address the first research problem, we define properties to characterize subgraphs (with respect to an extension) under different semantics (admissible, complete, grounded, preferred, and stable), and specify how the probability of a conflict-free set $E$ being an extension can be computed by using these properties. In Section 5,  to address the second research problem, algorithms are developed to evaluate the performance of the new approach (with a comparison to {an existing possible worlds based approach}). In Section 6, some computational properties of the new approach are briefly discussed. In Section 7, some existing work closely related to this paper is introduced and discussed. Finally, in Section 8, we conclude the paper and point out some future work.

This paper is a substantial extension of the paper introduced in \cite{Liao:LORI}. The {extension} mainly consists of the following aspects:
\begin{itemize}
\item {We reformulate} the approach of characterizing subgraphs of a PrAG with respect to a given extension, with a more detailed analysis of the properties used to characterize subgraphs;
\item {further} study the semantics of probabilistic argumentation by directly using properties for characterizing subgraphs;
\item {develop} algorithms to evaluate the efficiency of the new approach; {and}
\item {analyze} the computational properties of the new approach from the perspective parameterized complexity theory.
\end{itemize}

\section{Preliminaries} 
\subsection{Classical abstract argumentation}
The notions of (classical) abstract argumentation were originally introduced in \cite{Dung:AIJ} and then extended by many researchers (please refer to \cite{Baroni:KER} for an excellent introduction), including abstract argumentation framework (called \textit{argument graph}, or \textit{classical argument graph}, in this paper), extension-based semantics and labelling-based semantics. 

An argument graph is a directed graph $G = (A,R)$, in which $A$ is a set of nodes representing arguments and $R$ is a set of edges representing attacks between the arguments.

\begin{definition} \label{Def-AG}
An argument graph is a tuple $G = (A,R)$, where $A$ is a set of nodes representing arguments, and $R\subseteq A\times A$ is a set of edges representing attacks. %For convenience, sometimes we use $args(G)$ to denote $A$.
\end{definition}

As usual, we say that $\alpha\in A$ attacks $\beta\in A$ if and only if $(\alpha, \beta)\in R$. If $E\subseteq A$ and $\alpha\in A$ then we say that $\alpha$ attacks $E$ if and only if there exists $\beta\in E$ such that $\alpha$ attacks $\beta$, that $E$ attacks $\alpha$ if and only if there exists $\beta\in E$ such that $\beta$ attacks $\alpha$, and that $E$ attacks $E^\prime$ if and only if there exist $\beta\in E$ and $\alpha\in E^\prime$ such that $\beta$ attacks $\alpha$. Given $G = (A,R)$, for $\alpha\in A$ we write $\alpha^-_G$ for $ \{\beta\mid (\beta, \alpha)\in R\}$; for $E\subseteq A$ we write $E^-_G$ for $ \{\beta\mid \exists \alpha\in E: (\beta, \alpha)\in R\}$ and $E^+_G$ for $ \{\beta\mid \exists \alpha\in E: (\alpha, \beta)\in R\}$. Formally, we have the following formulas.
\begin{eqnarray} 
 \alpha^-_G&=& \{\beta\mid (\beta, \alpha)\in R\}\\
E^-_G&=&  \{\beta\mid \exists \alpha\in E: (\beta, \alpha)\in R\}\\
E^+_G&=& \{\beta\mid \exists \alpha\in E: (\alpha, \beta)\in R\}
\end{eqnarray}

If without confusion, we write $\alpha^-$, $E^-$ and $E^+$ for $\alpha^-_G$, $E^-_G$ and $E^+_G$ respectively. 

Given an argument graph, according to certain evaluation criteria, sets of arguments (called \textit{extensions}) are identified as acceptable together. Two important notions for the definitions of various kinds of extensions are \textit{conflict-freeness} and \textit{acceptability} of arguments.

\begin{definition}\label{Def-c-a}
Let $G=( A,R)$ be an argument graph, and $E\subseteq A$ be a set of arguments.
\begin{itemize}
  \item {$E$ is \emph{conflict-free} if and only if $\nexists \alpha, \beta\in E$, such that $(\alpha,\beta)\in R$.}
  \item {An argument $\alpha\in A$ is acceptable with respect to  (defended by) $E$, if and only if $\forall(\beta,\alpha)\in R$, $\exists\gamma\in E$, such that $(\gamma,\beta)\in R$.}
\end{itemize}
\end{definition}

Based on the above two notions, several classes of (classical) extensions can be defined as follows.

\begin{definition} \label{Def-G-s}
Let $G=( A,R)$ be an argument graph, and $E\subseteq A$ a set of arguments. 
\begin{itemize}
\item {$E$ is \emph{admissible} if and only if $E$ is conflict-free, and each argument in $E$ is acceptable with respect to $E$.}
 \item {$E$ is preferred if and only if $E$ is a maximal (with respect to set-inclusion) admissible set.}
  \item {$E$ is complete if and only if $E$ is admissible, and each argument that is acceptable with respect to $E$ is in $E$.}
 \item {$E$ is grounded if and only if $E$ is the minimal (with respect to set-inclusion) complete extension.}
  \item {$E$ is stable if and only if $E$ is conflict-free, and each argument in $A\setminus E$ is attacked by $E$.}
\end{itemize}
\end{definition}

In this paper, for convenience, we use $\sigma\in\{ad$, $co$, $pr$, $gr$, $st\}$ to represent a semantics (admissible, complete, preferred, grounded or stable). An extension under semantics $\sigma$ is called a $\sigma$-extension. The set of $\sigma$-extensions of $G$ is denoted as $\mathcal{E}_\sigma(G)$. In $G=(A,R)$, if $A=R=\emptyset$, then $\mathcal{E}_\sigma(G) = \{\emptyset\}$.

\begin{example}\label{ex-1}
Let $G_1 = (A_1, R_1)$ be an argument graph illustrated as follows.

  \begin{picture}(206,30)
 \put(0,15){\xymatrix@C=1cm@R=0.75cm{
  a\ar[r]&b \ar[l]\ar[r]&c\ar[r]&d\ar[l]\ar@(rd,ru)
   }}
    \end{picture}

According to Definition \ref{Def-G-s}, $G_1$ has four admissible sets: $\emptyset$, $\{a\}$, $\{b\}$ and $\{a,c\}$, in which $\emptyset$, $\{b\}$ and $\{a,c\}$ are complete extensions, $\{b\}$ and $\{a,c\}$ are preferred extensions, $\{a,c\}$ is the only stable extension, $\emptyset$ is the unique grounded extension. 
\end{example}

Corresponding to {the extension-based approach} introduced above,  {the \textit{labelling-based approach}} is another way to formulate argumentation semantics. Since we will use labelling-based approach to develop algorithms in Section 5.1, some basic notions of this approach are briefly introduced here.  
 The  idea underlying  {the labelling-based approach} is to give each argument a label, which is defined in advance. {In existing literature, the set of labels is usually defined as:} $\mathrm{IN}$, $\mathrm{OUT}$ and $\mathrm{UNDEC}$. The label $\mathrm{IN}$ indicates that the argument is explicitly accepted, the label $\mathrm{OUT}$ indicates that the argument is explicitly rejected, and the label $\mathrm{UNDEC}$ indicates that the status of the argument is undecided, meaning that one abstains from an opinion on whether the argument is accepted or rejected. Meanwhile, there could be some other choices for the set of labels. For instance, in \cite{JV99}, a four-valued labelling is considered. In this paper, we choose the three-valued-labelling, which can be formally defined as follows.

\begin{definition}[Labelling] \label{def-labeling}
Given an argument graph $G=( A,R)$ and three labels $\mathrm{IN}$, $\mathrm{OUT}$ and $\mathrm{UNDEC}$, a \textit{labelling} is a total function: 
\begin{eqnarray}
\mathcal{L}: A\mapsto \{\mathrm{IN}, \mathrm{OUT}, \mathrm{UNDEC}\}
 \end{eqnarray}
\end{definition}

Let $in(\mathcal{L}) = \{\alpha\mid \mathcal{L}(\alpha) = \mathrm{IN}\}$, $out(\mathcal{L}) = \{\alpha\mid \mathcal{L}(\alpha) = \mathrm{OUT}\}$, and $undec(\mathcal{L}) = \{\alpha\mid \mathcal{L}(\alpha) = \mathrm{UNDEC}\}$. A labelling $\mathcal{L}$ is often represented as a triple of the form $(in(\mathcal{L}) , out(\mathcal{L}) , undec(\mathcal{L}) )$.

One of criteria for labeling-based semantics is whether a label assigned to an argument is legal. According to Definition \ref{def-labeling},  given a labeling $\mathcal{L}$, the status assigned to each argument might not be legal. We say that assigning $\mathrm{IN}$ to an argument is legal if and only if all its attackers have been assigned $\mathrm{OUT}$; assigning $\mathrm{OUT}$ to an argument  is legal if and only if one of its attackers has been assigned $\mathrm{IN}$; and assigning $\mathrm{UNDEC}$ to an argument  is legal if and only if not all its attacks are labeled $\mathrm{OUT}$ and it does not have an attacker that is labeled $\mathrm{IN}$. Formally, we have the following definition.

\begin{definition}[Legal labelling] \label{legallabels}
{Let $\mathcal{L}$  be a labeling of an argument graph $G=( A,R)$ and $\alpha\in A$.} 
\begin{itemize}
\item  {$\alpha$ is legally $\mathrm{IN}$ if and only if $\mathcal{L}(\alpha) = $ $\mathrm{IN}$ and for all $\beta\in A$, if $(\beta, \alpha)\in R$, then $\mathcal{L}(\beta) =$ $\mathrm{OUT}$.}
\item  {$\alpha$ is legally $\mathrm{OUT}$ if and only if $\mathcal{L}(\alpha) = $ $\mathrm{OUT}$ and there exists $\beta\in A$, such that $(\beta, \alpha)\in R$, and $\mathcal{L}(\beta) =$ $\mathrm{IN}$.}
\item  {$\alpha$ is legally $\mathrm{UNDEC}$ if and only if $\mathcal{L}(\alpha) = $ $\mathrm{UNDEC}$ and \\
(1) there exists $\beta\in A$, such that $(\beta, \alpha)\in R$, and $\mathcal{L}(\beta) \neq$ $\mathrm{OUT}$, and \\
(2) for all $\beta\in A$, if $(\beta, \alpha)\in R$, then $\mathcal{L}(\beta) \neq$ $\mathrm{IN}$.}
\end{itemize}
\end{definition}

According to the notion of legal labelling, the notion of illegal labelling can be defined as follows.

\begin{definition}[Illegal labelling] \label{illegallabels}
{Let $\mathcal{L}$  be a labeling of an argument graph $G=( A,R)$ and $\alpha\in A$.} 
\begin{itemize}
\item {$\alpha$ is illegally $\mathrm{IN}$ if and only if $\mathcal{L}(\alpha) = $ $\mathrm{IN}$, but $\alpha$ is not legally $\mathrm{IN}$. }
\item {$\alpha$ is illegally $\mathrm{OUT}$ if and only if $\mathcal{L}(\alpha) = $ $\mathrm{OUT}$, but $\alpha$ is not legally $\mathrm{OUT}$. }
\item {$\alpha$ is illegally $\mathrm{UNDEC}$ if and only if $\mathcal{L}(\alpha) = $ $\mathrm{UNDEC}$, but $\alpha$ is not legally $\mathrm{UNDEC}$. }
\end{itemize}
\end{definition}

{Based on the notions of legal labelling, labeling-based semantics can be defined as follows. }

\begin{definition}[Labeling-Based Semantics]
{Let $\mathcal{L}$  be a labeling of an argument graph $G=( A,R)$.} 
\begin{itemize}
\item {$\mathcal{L}$ is an admissible labeling, if and only if each argument that is labeled $\mathrm{IN}$ is legally $\mathrm{IN}$, and each argument that is labeled $\mathrm{OUT}$ is legally $\mathrm{OUT}$.}
\item {$\mathcal{L}$ is a complete labeling, if and only if it is an admissible labeling, and each argument that is labeled $\mathrm{UNDEC}$ is legally $\mathrm{UNDEC}$.}
\item {$\mathcal{L}$ is a grounded labeling, if and only if it is a complete labeling, and $in(\mathcal{L})$ is minimal (with respect to set inclusion).}
\item {$\mathcal{L}$ is a preferred labeling, if and only if it is a complete labeling, and $in(\mathcal{L})$ is maximal (with respect to set inclusion).}
\item  {$\mathcal{L}$ is a stable labeling, if and only if it is a complete labeling, and $undec(\mathcal{L}) =\emptyset$. }
\end{itemize} 
\end{definition}

Based on the above notions, Modgil and Caminada developed algorithms (called MC algorithms)\cite{Modgil2009B} to compute the preferred labellings and the grounded labelling of an argument graph\footnote{It is worth to mention that in recent years, there are various approaches for computing the semantics of argumentation, including reduction approaches (e.g. the alrogithms based on ASP slovers) and direct approaches (e.g. the MC algorithms). Some of them have appeared in the International Competition on Computational Models of Argumentation (http://argumentationcompetition.org/2015/solvers.html). Since the choice of different implemention approaches does not basically affect the empirical results of our approach, for simplicity and without loss of generarity, we only introduce and exploit the MC algorithm for computing preferred labellings.}.

The MC algorithm for computing preferred labellings is realized by computing \textit{admissible labellings} that maximize the number of arguments that are legally IN. Here, admissible labellings are generated by starting with a labelling that labels all arguments IN and then iteratively, selects arguments that are illegally IN (or \textit{super-illegally IN}) and applies a \textit{transition step} to obtain a new labelling, until a labelling is reached in which no argument is illegally IN. In this algorithm, the notions of \textit{super-illegally IN} and \textit{transition step} are introduced as follows. For more details about the MC algorithms, please refer to  \cite{Modgil2009B}.

First, since all arguments are initially labelled IN, some of which might be illegal. To get an admissible labelling which might be a preferred labelling, it is necessary to change the label of each argument that is illegally IN, preferably without creating any arguments that are illegally OUT. {The notion of a transition step} is used for this purpose. In other words, a transition step basically takes an argument that is illegally IN and relabels it to OUT. It then checks if, as a result of this, one or more arguments have become illegally OUT. If this is the case, then these arguments are relabelled to UNDEC. Formally, the notion of \textit{transition step} is defined as follows \cite{Modgil2009B}.

\begin{definition}[Transition step]\label{transitionstep}
Let $\mathcal{L}$ be a labelling for $G=(A,R)$ and $\alpha$ be an argument that is illegally IN in $\mathcal{L}$. A transition step on $\alpha$ in $\mathcal{L}$ consists of the following:
\begin{itemize}
\item the label of $\alpha$ is changed from IN to OUT;
\item for every $\beta\in\{\alpha\}\cup\{\gamma\mid (\alpha,\gamma)\in R\}$, if $\beta$ is illegally OUT, then the label of $\beta$ is changed from OUT to UNDEC.
\end{itemize}
\end{definition} 

Second, if we select arbitrarily the arguments that are illegally IN to do transition steps, then we might obtain some admissible labellings that are not complete {labellings} (and therefore not preferred labellings). To improve the efficiency of computation, in the MC algorithm for preferred labellings, they proposed a notion, called \textit{super-illegally} IN. It is said that an argument $\alpha$ in $\mathcal{L}$ that is illegally IN, is also super-illegally IN if and only if it is attacked by an argument $\beta$ that is legally IN in $\mathcal{L}$, or UNDEC in $\mathcal{L}$. This notion can be used to guide the choice of arguments on which to perform transition steps, such that the non-complete labellings {can be avoided}. %Formally, we have the following definition \cite{Modgil2009B}.

\subsection{Probabilistic abstract argumentation} 
The notions of probabilistic abstract argumentation are defined by combining the notions of classical abstract argumentation and those of probabilistic theory, including probabilistic argument graph and its semantics. 

According to \cite{Hunter:comma}, we have the following definition.

\begin{definition}\label{def-prob-arg}
A probabilistic argument graph (or PrAG for short) is a triple $G^p = (A,R, p)$ where $G =(A,R)$ is an argument graph and $p: A\rightarrow [0,1]$ is a probability function assigning to every argument $\alpha\in A$ a probability $p(\alpha)$ that $\alpha$ appears (and hence a probability $1-p(\alpha)$ that $\alpha$ does not appear).
\end{definition}

In existing literature, the semantics of a PrAG is defined according to the notion of possible world. Given a  PrAG, a possible world represents a scenario consisting of some subset of the arguments and attacks in the graph. So, given a PrAG with $n$ nodes, there are up to $2^n$ subgraphs with nonzero probability. A subgraph induced by a set $A^\prime\subseteq A$ is represented as  $G^\prime= (A^\prime, R^\prime)$, in which $R^\prime = R\cap (A^\prime\times A^\prime)$. {For convenience, we also use $G_{\downarrow A^\prime}$ to denote a subgraph $G^\prime = (A^\prime, R^\prime)$.}
Under a semantics $\sigma\in \{ad, co, pr, gr, st\}$, the extensions of each subgraph are computed according to the definition of classical argumentation semantics. Then, the probability that a set of arguments $E\subseteq A$ is a $\sigma$-extension, denoted as $p(E^\sigma)$, is the sum of the probability of each subgraph for which $E$ is a $\sigma$-extension. 
In calculating the probability of each subgraph, we assume independence of arguments appearing in a graph.  A discussion about the assumption of independence of arguments is presented in Section 7.1.

For simplicity, let us abuse the notation, using $p(\bar{\alpha})$ to denote $1- p(\alpha)$. Then, the probability of subgraph $G^\prime$, denoted $p(G^\prime)$, can be defined as follows.
\begin{eqnarray}
p(G^\prime) &=& (\Pi_{\alpha\in A^\prime}\, p(\alpha))\times (\Pi_{\alpha\in A\setminus A^\prime}\, p(\bar{\alpha})) \label{formula-subgr}
\end{eqnarray}

Given a PrAG $G^p = (A,R, p)$, let $Q_\sigma(E)$ denote the set of subgraphs of $G$, each of which has an extension $E$ under a given semantics $\sigma\in\{ad, co$, $pr, gr$, $st\}$. Based on formula (\ref{formula-subgr}),   $p(E^\sigma)$ is defined as follows \cite{Hunter:comma}. 
\begin{eqnarray}
p(E^\sigma) &=& \Sigma_{G^\prime\in Q_\sigma(E)}\,p(G^\prime) \label{formula-2}
\end{eqnarray}

\begin{example}\label{ex-2}
Let $G_1^p = (A_1, R_1, p)$ be a PrAG (illustrated as follows), where $p(a) = 0.5$, $p(b) = 0.8$, $p(c) = 0.4$ and $p(d) = 0.5$. 

  \begin{picture}(206,40)
 \put(0,23){\xymatrix@C=1cm@R=0.75cm{
  a\ar[r]&b \ar[l]\ar[r]&c\ar[r]&d\ar[l]\ar@(rd,ru)
   }}
\small \put(0,9){0.5}
 \put(38,9){0.8}
\put(76,9){0.4}
\put(116,9){0.5}
    \end{picture}

The subgraphs of $G_1^p$ are presented in Table 1. 
\\
\begin{table}[!htp] \label{table-1}
\begin{center}
\renewcommand\arraystretch{1.5}
\begin{tabular}{@{\hspace{0.3cm}}l@{\hspace{0.9cm}}l@{\hspace{0.9cm}}l@{\hspace{0.9cm}}l@{\hspace{0.4cm}}}
  \hline
&\small Subgraphs &\small \parbox{2cm}{Probability of subgraph}&\small \parbox{1.3cm}{Preferred\\ extensions} \\
   \hline
$G^1_1$&$a\leftrightarrow b\rightarrow c\leftrightarrow d\righttoleftarrow$&0.08&$\{b\}, \{a, c\}$\\
 \hline
$G^2_1$&$a\leftrightarrow b\rightarrow c$&0.08&$\{b\}, \{a,c\}$\\
 \hline
$G^3_1$&$a\leftrightarrow b$\hspace{0.45cm} $d\righttoleftarrow$&0.12&$\{a\}, \{b\}$ \\
 \hline
$G^4_1$&$a\leftrightarrow b$&0.12&$\{a\},\{b\}$ \\
 \hline
$G^5_1$&$a  \hspace{0.45cm}c\leftrightarrow d\righttoleftarrow$&0.02&$\{a,c\}$\\
 \hline
$G^6_1$&$a \hspace{0.45cm}c$&0.02&$\{a,c\}$\\
 \hline
$G^7_1$&$a\hspace{0.45cm} d\righttoleftarrow$&0.03&$\{a\}$ \\
 \hline
$G^8_1$&a &0.03&$\{a\}$ \\
 \hline
$G^9_1$&$b\rightarrow c\leftrightarrow d\righttoleftarrow$&0.08&$\{b\}$\\
 \hline
$G^{10}_1$&$b\rightarrow c$&0.08&$\{b\}$ \\
 \hline
$G^{11}_1$&$b$\hspace{0.45cm} $d\righttoleftarrow$&0.12&$\{b\}$ \\
 \hline
$G^{12}_1$&$b$&0.12&$\{b\}$\\
 \hline
$G^{13}_1$&$c\leftrightarrow d\righttoleftarrow$&0.02&$\{c\}$\\
 \hline
$G^{14}_1$&$c$&0.02&$\{c\}$ \\
 \hline
$G^{15}_1$& $d\righttoleftarrow$ &0.03&$\{\}$ \\
 \hline
$G^{16}_1$&&0.03&$\{ \}$\\
\hline
%\vspace{0.01cm}
\end{tabular}
\caption{Subgraphs of $G_1$}
\end{center}
\end{table}
\end{example}

According to formula (\ref{formula-2}),  there are 5 preferred extensions with non-zero probability: 
\begin{eqnarray*}
p(\emptyset^{pr}) &=&p(G^{15}_{1})+p(G^{16}_{1})= 0.06 \\
p(\{a\}^{pr}) &=& p(G^3_1)+p(G^4_1)+p(G^7_1)+p(G^8_1)=0.3 \\
p(\{b\}^{pr}) &=&p(G^1_1)+p(G^2_1)+p(G^3_1)+p(G^4_1) + p(G^9_1)+p(G^{10}_1)\\
&&+p(G^{11}_1)+p(G^{12}_1)=0.8   \\
p(\{c\}^{pr}) &=& p(G^{13}_1)+p(G^{14}_1)=0.04 \\
p(\{a, c\}^{pr}) &=&p(G^1_1)+p(G^2_1)+p(G^5_1)+p(G^6_1)= 0.2
\end{eqnarray*}

 {This example shows that by using the existing possible worlds based approach, in order to obtain the probability that a set $E$ of arguments is an extension under a given semantics (i.e., $p(E^\sigma)$), one has to compute the extensions  of all subgraphs under this semantics, although some of these subgraphs have no extension $E$. Since many irrelevant subgraphs are constructed and computed, and in many cases, computing extensions of subgraphs is computationally intractable,  this possible worlds based approach is fundamentally inefficient or infeasible. In \cite{Fazzinga:sum13}, Fazzinga et al proposed a new approach and proved that under admissible and stable semantics, the problem of determining
$p(E^\sigma)$  is tractable.  However, under complete, grounded and preferred semantics, the problem of determining $p(E^\sigma)$ is $FP^{\sharp P}$-complete.} This calls for developing more efficient approaches, including the approximate approaches introduced in  \cite{Li:TAFA} and \cite{Fazzinga:sum13}.

\section{Characterized subgraphs with respect to an extension}
Given a PrAG, since the probability of a set of arguments $E$ being an extension under a given semantics $\sigma$ (i.e. $p(E^\sigma)$) is equal to the sum of the probabilities of the subgraphs each of which has an extension $E$, the main issue  is to identify the subgraphs.  As mentioned in Section 1, unlike the existing approaches, we define general properties  to characterize the subgraphs, such that $p(E^\sigma)$ can be computed by using these properties, rather than by blindly constructing and computing all subgraphs of the PrAG. %some irrelevant subgraphs are excluded and some computation of extensions can be avoided. 

%we introduce a new approach to formulate the semantics of probabilistic argumentation by characterizing {subgraphs with respect to an extension}, 
%
%
%such that the probability of a set of arguments being an extension can be evaluated without computing (or with less computation of) the extensions of subgraphs.

To begin with, let us introduce a notion of \textit{$\sigma$-subgraph with respect to an extension}: If a {subgraph}  has a $\sigma$-extension $E$, then it is called a $\sigma$-subgraph with respect to $E$. Formally, we have the following definition.

\begin{definition}\label{def-5}
Let $G^p = (A,R, p)$ be a PrAG, $G = (A,R)$ be the corresponding classical argument graph, $G_{\downarrow A^\prime}$ be a subgraph of $G$ where $A^\prime\subseteq A$, and $E\subseteq A$ be a set of arguments. We say that $G_{\downarrow A^\prime}$ is a $\sigma$-subgraph of $G$ with respect to $E$, if and only if $G_{\downarrow A^\prime}$ has a $\sigma$-extension $E$, where $\sigma\in\{ad, co, pr, gr, st\}$.
\end{definition}

%When context is clear, the reference to $E$ of a $\sigma$-subgraph of $G$ is sometimes dropped.  

\begin{example}
Consider $G_1^p$ in Example \ref{ex-2}. Given $E_1 = \{a\}$, $G_1^3, G_1^4, G_1^7$ and $G_1^8$ are preferred subgraphs of $G_1^p$ with respect to $E_1$.
\end{example}

Then, given a PrAG $G^p = (A,R, p)$, a set of arguments $E\subseteq A$ and a semantics $\sigma\in\{ad, co, pr, gr, st\}$, a function (called \textit{subgraph identification function}) is used to map $E$ to a set of $\sigma$-subgraphs of $G$ with respect to $E$. 

\begin{definition} \label{def-characterization}
Let $G^p = (A,R, p)$ be a PrAG, and $G = (A,R)$ be the corresponding argument graph. Let $\mathbb{G} = \{G_{\downarrow A^\prime} \mid A^\prime \in 2^A\}$ be the set of all subgraphs of $G$. A \textit{subgraph identification function} under a given semantics $\sigma\in\{ad, co, pr, gr, st\}$ (denoted as $\rho^\sigma$) is defined as a mapping:
\begin{equation}
\rho^\sigma: 2^A \rightarrow 2^{\mathbb{G}}
\end{equation}
such that given $E\in 2^A$, for all $G^\prime\in \rho^\sigma(E)$, $G^\prime$ is a $\sigma$-subgraph of $G$ with respect to $E$.
\end{definition}

In Definition \ref{def-characterization}, $\rho^\sigma$ ($\sigma\in\{ad, co, pr, gr, st\}$) can be understood as a class of functions (i.e., $\rho^{ad}$, $\rho^{co}$, $\rho^{pr}$, $\rho^{gr}$ and $\rho^{st}$), each of which is a function under a given semantics.  

These subgraph identification functions can be instantiated in different ways. A simple but inefficient way is to construct the set of all subgraphs of $G$ (i.e., $\mathbb{G}$) and then for each subgraph to verify whether it has an extension $E$. In terms of this approach, subgraphs are constructed blindly, although many of them are irrelevant. And, for each subgraph, the algorithm to verify $E$ being an extension of the subgraph might be intractable (e.g., under preferred semantics, the problem of verifying whether $E$ is an extension is {coNP-complete \cite{Dunne2009}). }

To cope with this problem, we introduce as follows another way to instantiate the subgraph identification functions. In this new approach,  properties related to $E$ are used to characterize the set of subgraphs each of which has an extension $E$. Since  $p(E^\sigma)$ may be computed by using these properties (please refer to Section 4 for details), the characterized subgraphs can be kept (completely or partially) implicit, rather than explicitly constructed and computed  (although the subgraphs may be also explicitly represented according to the properties, as presented in formulas (\ref{for-admiss})-(\ref{formalus-gr})).

{Since when $E$ is not conflict-free, the set of characterized subgraphs with respect to $E$ is an empty set, for simplicity, when talking about the set of characterized subgraphs with respect to a set of arguments $E$, we only consider the cases where $E$ is conflict-free.}

Under different semantics, properties used to characterize subgraphs may vary. However, they are all based on the following components related to $E$ (as illustrated in Figure 1)\footnote{The slices of the pie are used to indicate different components related to $E$. Their sizes  are not important.}:
\begin{itemize}
\item[1.] $E$;
\item[2.] $E_G^-\setminus E_G^+$: the set of arguments each of which attacks $E$ but is not attacked by $E$; 
\item[3.] $E_G^+$: the set of arguments each of which is attacked by $E$;
\item[4.] $I = A\setminus (E\cup E_G^+\cup E_G^-)$: the set of arguments each of which is not in $E$, $E_G^+$ or $E_G^-$. We call $I$ the set of remaining arguments (of $G$ with respect to $E$) that indirectly affects $E$ being a $\sigma$-extension. 
\end{itemize} 

\begin{figure}[!htp]
%  \begin{picture}(206,180)
\begin{center}
\includegraphics [width=0.6\textwidth]{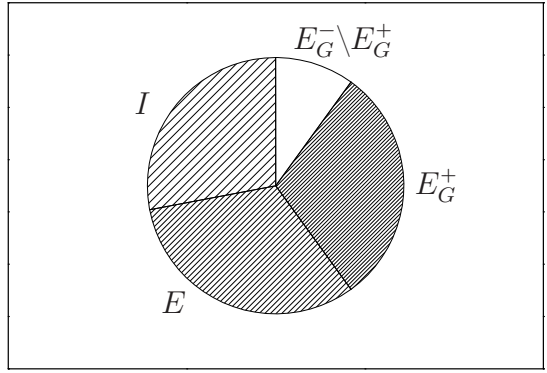}
 \put(-64,80){$E_G^+$}
 \put(-110,135){$E_G^-\setminus E_G^+$}
 \put(-160,35){$E$}
 \put(-170,110){$I$}
%\put(-185,110){\circle{45}}
\end{center}
% \end{picture}
\caption{Four components of $G$ w.r.t. $E$. It holds that $E\cup (E_G^-\setminus E_G^+) \cup E_G^+\cup I = A$.}
\end{figure}

%As mentioned above, the basic motivation of defining notions of $\sigma$-subgraph and subgraph characterization function  is to formulate a new approach to identify the set of subgraphs with respect to a given extension.  Given a PrAG and a conflict-free set $E$ of arguments, according to existing approaches (e.g., the Monte-Carlo simulation approach), subgraphs are constructed blindly, although many of them are irrelevant. And, for each subgraph, whether $E$ is an extension is verified by an algorithm which might be intractable (e.g., under preferred semantics, the algorithm is coNP-complete). %So, a natural question whether it is possible to identify the subgraphs in a manner of backward chaining (or goal oriented) according to some properties of the subgraphs.
%To cope with this problem, under a given semantics, we define properties to instantiate the subgraph characterization function such that most subgraphs are not necessary to be explicitly constructed and computed. The key insight behind this approach is that given a PrAG, a conflict set $E$ of arguments and a  semantics $\sigma$, whether $E$ is a $\sigma$-extension of a subgraph is determined by exploiting the information of the PrAG with respect to $E$ (encoded by $E$, $E_G^-\setminus E_G^+$, $E_G^+$, and $I = A\setminus (E\cup E_G^+\cup E_G^-)$), not by blindly constructing all subgraphs of the PrAG and verifying whether $E$ is a $\sigma$-extension of each subgraph. 

Firstly, under admissible semantics,  each admissible subgraph can be characterized by the following two properties (as illustrated in Figure 2):
\begin{description}
\item[Prop1: ]All arguments in $E$ appear in the subgraphs; and
\item[Prop2: ]All arguments in $E_G^-\setminus E_G^+$ do not appear in the subgraph (while the appearance of arguments in any subset of $I \cup E_G^+$ does not affect $E$ being an extension of the subgraph). 
\end{description}

Prop2 means that every argument in $E$ is acceptable with respect to $E$. Given that $E$ is conflict-free and every argument in $E$ is acceptable with respect to $E$, $E$ is an admissible extension. So, by definition, the subgraph is an admissible  subgraph. 

\begin{figure}[!htp]
%  \begin{picture}(206,180)
\begin{center}
\includegraphics [width=0.6\textwidth]{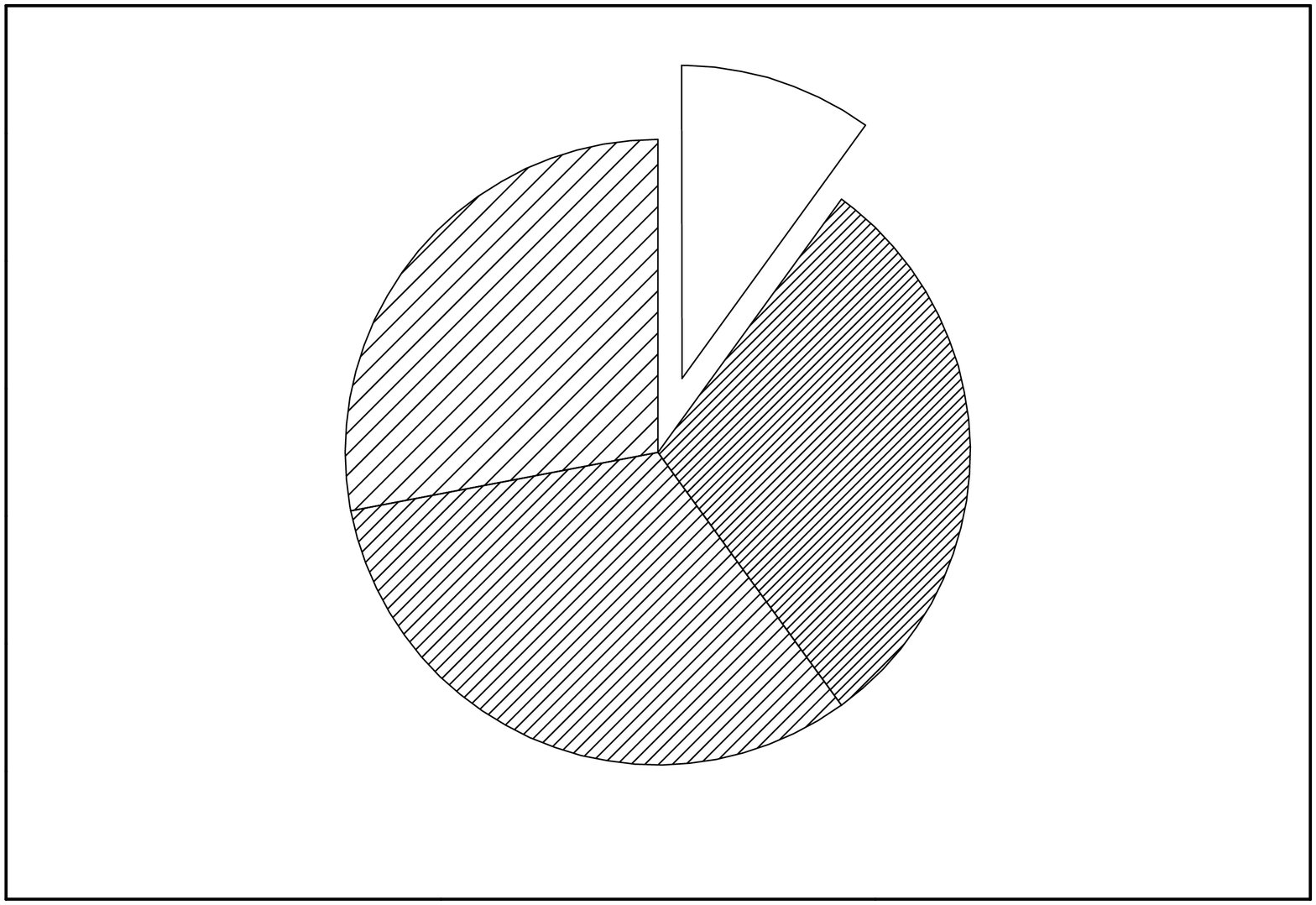}
 \put(-64,80){$E_G^+$}
 \put(-82,135){$E_G^-\setminus E_G^+$}
 \put(-160,35){$E$}
 \put(-170,110){$I$}
%\put(-185,110){\circle{45}}
\end{center}
% \end{picture}
\caption{Given a PrAG $G = (A, R, p)$, a subgraph $G^\prime$ is an admissible subgraph w.r.t. $E$ if and only if arguments in $E$ appear, arguments in $E_G^-\setminus E_G^+$ do not appear, while arguments in any subset of $I \cup E_G^+$ may apear in the subgraph.}
\end{figure}

According to the above analysis,  we have the following proposition.

\begin{proposition}\label{the-1}
Let $G^p = (A,R, p)$ be a PrAG, $G=(A, R)$ be a corresponding argument graph,  and $E\subseteq A$ be a conflict-free set of arguments. Then, for all $B\in 2^{I\cup E_G^+}$, $G_{\downarrow E\cup B}$ is an admissible subgraph of $G$ with respect to $E$.%,
%
%$G_{\downarrow A^\prime}$ is  an admissible subgraph of $G$ with respect to $E$, if and only if the following properties hold:
%\begin{itemize}
%\item $E\subseteq {A}^\prime$, which means that all arguments in $E$ appear in $G^\prime$; and
%\item $(E_G^-\setminus E_G^+)\cap {A}^\prime = \emptyset$, which means that every argument in $E_G^-\setminus E_G^+$ does not appear in $G^\prime$.
%\end{itemize}
\end{proposition}

\begin{proof}
We need to verify that  $E$ is an admissible set of $G_{\downarrow E\cup B}$. Since $E$ is conflict-free, we only need to prove that $\forall \alpha \in E$, $\alpha$ is acceptable with respect to $E$. Since $E$ is conflict-free and there is no interaction between $I$ and $E$, it holds that $\alpha$ is not attacked by the arguments in $E\cup (I\cap B)$.  And, $\forall\beta\in E_G^+\cap B$, no matter whether $\beta$ attacks $\alpha$, by the definition of $E_G^+$, $\beta$ is attacked by $E$. In summary, $\forall \beta \in E\cup (I\cap B)\cup (E_G^+\cap B) = E\cup B$, $\beta$ either does not attack $\alpha$ or is attacked by $E$. So, $\forall \alpha \in E$, $\alpha$ is acceptable with respect to $E$.
\end{proof}

According to Proposition \ref{the-1}, the set of admissible subgraphs $\rho^{ad}(E)$ can be specified as follows:
\begin{equation} \label{for-admiss}
\rho^{ad}(E) = \{G_{\downarrow E\cup B} \mid B\in 2^{I\cup E_G^+}\}
\end{equation}

\begin{example} \label{ex-ad}
Consider $G_1^p$ in Example \ref{ex-2} again. According to formula (\ref{for-admiss}), there are eight admissible subgraphs with respect to $\{a\}$: $G^1_1$, $G^2_1$, $\dots$, $G^8_1$ (as shown in the third column of Table 2), i.e., $\rho^{ad}(\{a\}) = \{G^1_1, G^2_1, \dots, G^8_1\}$.

\begin{table}[!htp] \label{table-2}
\begin{center}
\renewcommand\arraystretch{1.5}
\begin{tabular}{@{\hspace{0.1cm}}l@{\hspace{0.2cm}}l@{\hspace{0.3cm}}l@{\hspace{0.2cm}}l@{\hspace{0.2cm}}l@{\hspace{0.2cm}}l@{\hspace{0.2cm}}l@{\hspace{0.1cm}}}
  \hline
&\small subgraph &\small \parbox{1.7cm}{admissible subgraph w.r.t. $\{a\}$}&\small \parbox{1.7cm}{comple subgraph w.r.t. $\{a\}$}&\small \parbox{1.7cm}{stable subgraph w.r.t. $\{a\}$}&\small \parbox{1.7cm}{preferred subgraph w.r.t. $\{a\}$} &\small \parbox{1.7cm}{grounded subgraph w.r.t. $\{a\}$}   \\
   \hline
$G^1_1$&$a\leftrightarrow b\rightarrow c\leftrightarrow d\righttoleftarrow$&Yes&Yes&No&No&No\\
 \hline
$G^2_1$&$a\leftrightarrow b\rightarrow c$&Yes&No&No&No&No\\
 \hline
$G^3_1$&$a\leftrightarrow b$\hspace{0.45cm} $d\righttoleftarrow$&Yes&Yes&No&Yes&No \\
 \hline
$G^4_1$&$a\leftrightarrow b$&Yes&Yes&Yes&Yes&No \\
 \hline
$G^5_1$&$a  \hspace{0.45cm}c\leftrightarrow d\righttoleftarrow$&Yes&Yes&No&No&Yes\\
 \hline
$G^6_1$&$a \hspace{0.45cm}c$&Yes&No&No&No&No\\
 \hline
$G^7_1$&$a\hspace{0.45cm} d\righttoleftarrow$&Yes&Yes&No&Yes&Yes\\
 \hline
$G^8_1$&a &Yes&Yes&Yes&Yes&Yes\\
 \hline
$G^9_1$&$b\rightarrow c\leftrightarrow d\righttoleftarrow$&No&No&No&No&No\\
 \hline
$G^{10}_1$&$b\rightarrow c$&No&No&No&No&No\\
 \hline
$G^{11}_1$&$b$\hspace{0.45cm} $d\righttoleftarrow$&No&No&No&No &No\\
 \hline
$G^{12}_1$&$b$&No&No&No&No&No\\
 \hline
$G^{13}_1$&$c\leftrightarrow d\righttoleftarrow$&No&No&No&No&No\\
 \hline
$G^{14}_1$&$c$&No&No&No&No&No \\
 \hline
$G^{15}_1$& $d\righttoleftarrow$ &No&No&No&No &No\\
 \hline
$G^{16}_1$&&No&No&No&No&No\\
\hline
\vspace{0.01cm}
\end{tabular}
\caption{$\sigma$-subgraphs of $G_1$ with respect to $\{a\}$}
\end{center}
\end{table}
\end{example}

%{Second, since every stable extension is a complete extension, and under stable semantics no argument is undecided, we may infer that a complete subgraph is a stable subgraph (with respect to a set of arguments $E$), if and only if all arguments in the subgraph are included in $E\cup E^+$. We directly have the following theorem. (redefine this part)}

Secondly, under stable semantics,  each stable subgraph can be characterized by the following two properties (as illustrated in Figure 3):

%Based on the properties for characterizing admissible subgraphs, let us consider how to characterize a stable subgraph. According to Definition \ref{Def-G-s}, $E$ is a stable extension of a subgraph $G_{\downarrow A^\prime}$ if and only if $E\subseteq A^\prime$ is conflict-free, and each argument in $A^\prime\setminus E$ is attacked by $E$. So,  under stable semantics, given a PrAG $G=(A, R, p)$, a conflict-free set $E$ of arguments, all subgraphs having the following two properties can be included, while others are excluded (as illustrated in Fig. 2):
\begin{description}
\item[Prop1: ]All arguments in $E$ appear in the subgraph. 
\item[Prop3: ]All arguments in $A\setminus (E\cup E_G^+) = I \cup (E_G^-\setminus E_G^+)$ do not appear in the subgraph. 
\end{description}

\begin{figure}[!htp]
%  \begin{picture}(206,180)
\begin{center}
\includegraphics [width=0.6\textwidth]{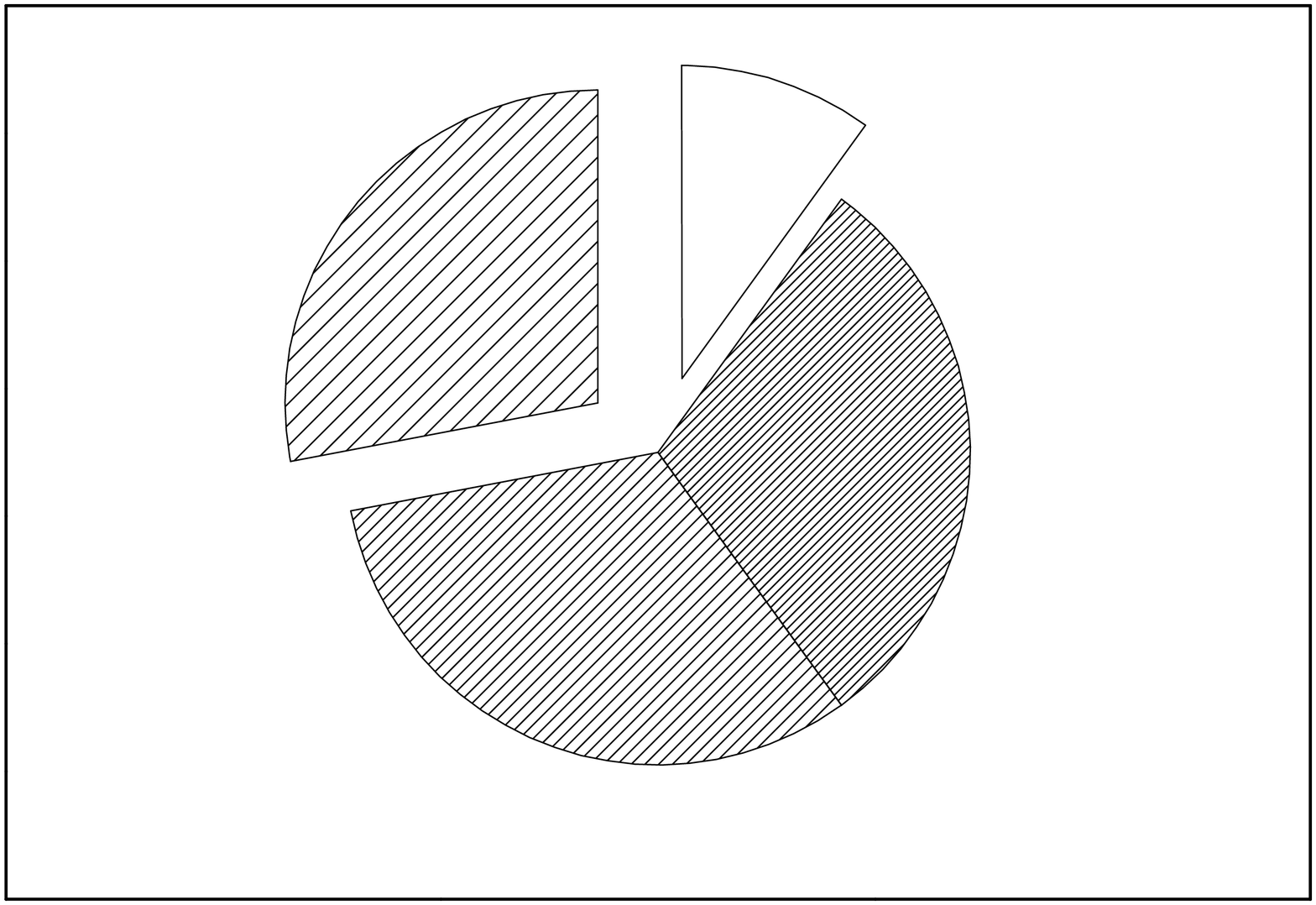}
 \put(-64,80){$E_G^+$}
 \put(-82,135){$E_G^-\setminus E_G^+$}
 \put(-160,35){$E$}
 \put(-175,119){$I$}
%\put(-185,110){\circle{45}}
\end{center}
% \end{picture}
\caption{Given a PrAG $G = (A, R, p)$, a subgraph $G^\prime$ is a stable subgraph w.r.t. $E$ if and only if arguments in $E$ appear, arguments in $I \cup (E_G^-\setminus E_G^+)$ do not appear, while arguments in any subset of $E_G^-$ may appear.}
\end{figure}

Prop3 means that for each argument  $\alpha$ in $A\setminus E$, if it appears in the subgraph, then it is attacked by $E$ (i.e., $\alpha\in E_G^+$). Given that $E$ is conflict-free and for every argument  that is not in $E$ it is attacked by $E$,  $E$ is a stable extension. So, by definition, the subgraph is a stable  subgraph. 

Formally, we have the following proposition.
%
%Second, under stable semantics, given a PrAG, a conflict-free set $E$ of arguments, a subgraph induced by a set $A^\prime\subseteq A$ is a stable subgraph if and only if $A^\prime = E\cup E_{A^\prime}^-$, i.e., all arguments in $E$ appear, a subset $B\in 2^{E^+}$ appear, and other arguments do not appear. Formally, we have the following theorem.

\begin{proposition}\label{the-st}
Let $G^p = (A,R, p)$ be a PrAG, $G=(A, R)$ be a corresponding argument graph, and $E\subseteq A$ be a conflict-free set of arguments. Then, for all $B\in 2^{E_G^+}$, $G_{\downarrow E\cup B}$ is a stable subgraph of $G$ with respect to $E$. 
%$G_{\downarrow A^\prime}$ is  a stable subgraph of $G$ with respect to $E$, if and only if the following properties hold:
%\begin{itemize}
%\item $E\subseteq {A}^\prime$, which means that all arguments in $E$ appear in $G^\prime$; and
%\item $A^\prime\subseteq E\cup E_G^+ $, which means that all arguments in $A\setminus (E\cup E_G^+) = I \cup (E_G^-\setminus E_G^+)$ do not appear in the subgraph. .
%For all $\alpha\in A^\prime\setminus E$, $\alpha^-\cap E\neq \emptyset$.
%\end{itemize}
\end{proposition}

\begin{proof}
Since $E$ is conflict-free, to prove $E$ being a stable extension of $G_{\downarrow E\cup B}$, we only need to verify that $\forall \alpha\in (E\cup B)\setminus E = B$, $\alpha$ is attacked by $E$. Since $\alpha\in B \subseteq E_G^+$, by the definition of $E_G^+$, $\alpha$ is attacked by $E$.
\end{proof}

According to Proposition \ref{the-st}, the set of stable subgraphs $\rho^{st}(E)$ can be specified as follows:
\begin{equation} \label{for-stable}
\rho^{st}(E) = \{G_{\downarrow E\cup B} \mid B\in 2^{E_G^+}\}
\end{equation}

%By Theorem \ref{the-st} and formula (\ref{for-stable}), a stable subgraph $G_{\downarrow A^\prime}$ can also be characterized with two properties related to the appearance of arguments. So, the probability of $E$ being a stable extension can be directly computed without constructing and computing the subgraphs. 

Thirdly, under other semantics (complete, grounded and preferred), the set of remaining arguments $I = A\setminus (E\cup E_G^+\cup E_G^-)$ plays a very important role in characterizing $\sigma$-subgraphs. 

Let $G_{\downarrow E\cup B}$ (where $B\in 2^{I\cup E_G^+}$) be an admissible subgraph of $G$ with respect to $E$, and $B^\prime = B\cap I$. Whether $G_{\downarrow E\cup B}$ is a complete subgraph with respect to $E$ is determined by a property of the arguments in $B^\prime\in 2^I$. Intuitively, if the following property holds, then $G_{\downarrow E\cup B}$ is a complete subgraph with respect to $E$  (as illustrated in Figure 4):

%Now, based on the properties for characterizing admissible, we may define properties to characterize the subgraphs under other semantics.

%According to the relationship between complete extension and admissible extension, it holds that with respect to a given conflict-free set of arguments, every complete subgraph is an admissible subgraph, but not vice versa. However, in an admissible subgraph with respect to a set $E$,  if all arguments acceptable with respect to $E$ is in $E$, then the  admissible subgraph is a complete subgraph, in which $E$ is a complete extension of the subgraph. So,   given a PrAG $G^p=(A, R, p)$ and an admissible subgraph $G_{\downarrow A^\prime}$ with respect to  a conflict-free set $E$ of arguments,   it is a complete subgraph w.r.t $E$ if and only if the following additional property holds (as illustrated by Fig. 3):
\begin{description}
%\item[Prop1: ] It is a . 
\item[Prop4: ]For all  $\alpha\in B^\prime$, $\alpha$ is attacked by $B^\prime$. 
\end{description} 

This property means that for every remaining argument $\alpha\in B^\prime$, $\alpha$ is not acceptable with respect to $E$. Based on this property, we have the following proposition.

%Prop4 means that for each argument $\alpha$ which is not in $E$, $E_G^+$ or $E_G^-$, it is attacked by some argument $\beta$ that is also not in $E$, $E_G^+$ or $E_G^-$. As a result, $\alpha$ is not acceptable with respect to $E$. Otherwise, $\beta$ is attacked by $E$ and therefore  $\beta\in E_G^+$, which contradicts the assumption that $\beta$ is not in $E_G^+$. Given that $E$ is an admissible extension of the subgraph (according to Prop1 and Prop2) and no other argument appearing in the subgraph is acceptable with respect to $E$ (according to Prop3) , $E$ is a complete extension. So, by definition, the subgraph is a complete subgraph.

\definecolor{grey}{rgb}{0.98,0.98,1}
\definecolor{grey2}{rgb}{0.95,0.95,1}

\begin{figure}[!tp]
%  \begin{picture}(206,180)
\begin{center}
\includegraphics [width=0.8\textwidth]{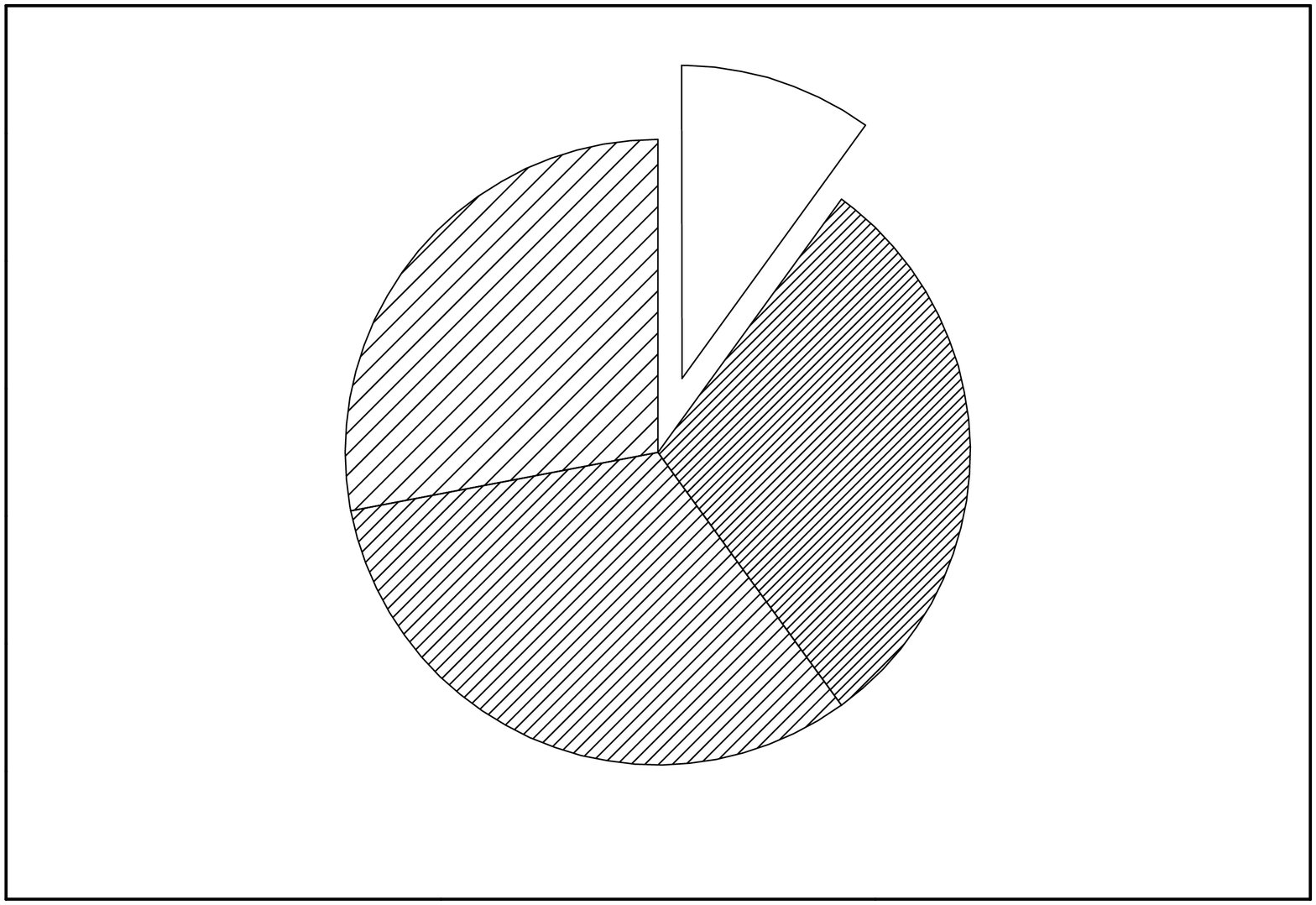}
 \put(-84,110){$E_G^+$}
 \put(-120,187){$E_G^-\setminus E_G^+$}
 \put(-190,35){$E$}
 \put(-225,150){$I$}
\put(-185,130){\textcolor{grey}{\circle*{35}}}
\put(-190,126){$B^\prime$}
\put(-120,110){\textcolor{grey2}{\circle*{25}}}
\put(-125,106){$B^{\prime\prime}$}
\end{center}
% \end{picture}
\caption{Given a PrAG $G = (A, R, p)$, a subgraph $G^\prime$ is a complete subgraph w.r.t. $E$ if and only if it is an admissible subgraph w.r.t. $E$, and each argument in $B^\prime$ is attacked by some arguments in $B^\prime$.}
\end{figure}

\begin{proposition} \label{theorem-com}
Let $G^p = (A,R, p)$ be a PrAG, $G=(A, R)$ be a corresponding argument graph, $E\subseteq A$ be a conflict-free set of arguments. For all $B\in 2^{I\cup E_G^+}$, $G_{\downarrow E\cup B}$ is a complete subgraph of $G$ with respect to $E$, if and only if %Let $G^{\prime\prime} = (A^{\prime\prime}, R^{\prime\prime})$, where $A^{\prime\prime} = args(G^\prime) \setminus (E\cup E^+)$ and $R^{\prime\prime} =R^\prime\cap(A^{\prime\prime}\times A^{\prime\prime})$. 
%Let $B^\prime = I \cap B$.
%$G_{\downarrow E\cup B}$ is a complete subgraph of $G$ with respect to $E$ if and only if the following property holds: 
%\begin{itemize}
%\item $E\subseteq {A}^\prime$;
%\item $(E_G^-\setminus E_G^+)\cap {A}^\prime = \emptyset$;
 $\forall \alpha\in  B^\prime$, $\alpha_{G}^- \cap B^\prime\neq \emptyset$. 
%\end{itemize}
\end{proposition}

\begin{proof}
Since $B\in 2^{I\cup E_G^+}$, according to Proposition  \ref{the-1},  $G_{\downarrow E\cup B}$ is an admissible subgraph.

%, and $\forall \alpha\in B^\prime$, $\alpha_{G}^-\cap B^\prime \neq \emptyset$, we have the following proof.

($\Rightarrow$:) When $G_{\downarrow E\cup B}$ is a complete subgraph of $G$ with respect to $E$, assume that $\exists \alpha\in  B^\prime$ such that $\alpha_G^- \cap B^\prime= \emptyset$. It follows that $\alpha$ is acceptable with respect to $E$, and therefore $E$ is not a complete extension, contradicting $G_{\downarrow E\cup B}$ is a complete subgraph with respect to $E$. 

($\Leftarrow$:)
{For all $\alpha\in  B^\prime$, for all $\beta\in \alpha_{G}^-\cap B^\prime$, $\beta$ can not be attacked by the arguments in $E$. Otherwise, $\beta$ is in $E_G^+$, contradicting $\beta \in B^\prime \subseteq I$ and $I \cap E_G^+ = \emptyset$. Since $\alpha_{G}^-\cap B^\prime \neq \emptyset$, $\alpha$ is not acceptable with respect to $E$.  Since $E$ is an admissible set}, $E$ is a complete extension. According to Definition \ref{def-5},  $G_{\downarrow E\cup B}$ is a complete subgraph of $G$ with respect to $E$.
\end{proof}

According to Proposition \ref{theorem-com}, the set of complete subgraphs $\rho^{co}(E)$ can be specified as follows:
\begin{equation} \label{formula-compl}
\rho^{co}(E) = \{G_{\downarrow E\cup B}  \mid (B\in 2^{I\cup E_G^+})\wedge(\forall \alpha\in B^\prime: \alpha_{G}^-\cap B^\prime\neq\emptyset)\}
\end{equation}

\begin{example} \label{ex-co}
%Continue Example \ref{ex-ad}. We have $I = \{c,d\}$, $$. 
%
Among the eight admissible subgraphs, except $G_1^2$ and $G_1^6$, others are complete subgraphs with respect to $\{a\}$  (as shown in the fourth column of Table 2), i.e., $\rho^{co}(\{a\}) =   \{G^1_1, G^3_1,G^4_1, G^5_1, G^7_1,  G^8_1\}$.

 With regard to ${G_1^2}$, $B^\prime = \{c\}$. Since $c_{G_1}^- =  \emptyset$,  $G_1^2$ is not a complete subgraph with respect to $\{a\}$. Similarly, $G_1^6$ is not a complete subgraph with respect to $\{a\}$. 
\end{example}

%\begin{theorem} \label{th-st}
%Let $G^p = (A,R, p)$ be a PrAG, $E\subseteq A$ be a conflict-free set of arguments, and $G^\prime = (A^\prime, R^\prime)$ be a complete subgraph of $G^p$ with respect to $E$. Then, ${G}^\prime$ is a stable subgraph of $G^p$ with respect to $E$ if and only if the following condition holds: $E\cup E^+ = args(G^\prime)$.
%\end{theorem}
%
%According to theorem \ref{th-st}, the function $\rho^{st}$ is specified as follows:
%\begin{equation} \label{formula-st}
%\rho^{st}(E) = \{G^\prime\in \rho^{co}(E)  \mid E\cup E^+ = args(G^\prime)\}
%\end{equation}
%
%
%\begin{example}
%Continue Example \ref{ex-co}. According to Theorem \ref{th-st}, it is not dif and only ificult to verify that among the six admissible subgraphs, only $G_1^4$ and $G_1^8$ are stable subgraphs with respect to $\{a\}$  (as shown in the fifth column of Table 2), i.e., $\rho^{st}(\{a\}) =   \{G^4_1,  G^8_1\}$.
%\end{example}
Then, under preferred semantics, given a complete subgraph $G_{\downarrow E\cup B}$ (where $B\in 2^{I\cup E_G^+}$), whether $G_{\downarrow E\cup B}$ is a preferred subgraph is determined by a property of the subgraph induced by $B^\prime\in 2^I$. More specifically, if the following property holds, then $G_{\downarrow E\cup B}$ is a preferred subgraph:

%The above theorems and formulas show that under admissible, complete, and stable semantics, the set of subgraphs with respect to an extension can be identified without computing the extensions of subgraphs. However, under preferred and grounded semantics, partial computation of extensions is needed.  

%More specifically, under preferred semantics, it is required that $E$ should be a maximal complete extension of the subgraph. In other words, no arguments in $I$ that appear in the subgraph can be added to $E$ to obtain a set $E^\prime$ such that $E^\prime$ is complete extension of the subgraph, and $E\subset E^\prime$. Since preferred semantics satisfies directionality property, we may decompose the subgraph into two components: one is induced by $B^\prime$, another is induced by $E\cup (B\setminus B^\prime)$. 

%So,   given a PrAG $G^p=(A, R, p)$ and a complete subgraph $G_{\downarrow E\cup B}$ with respect to  a conflict-free set $E$ of arguments,   it is a preferred subgraph w.r.t $E$ if and only if the following additional property holds:
\begin{description}
%\item[Prop1: ] It is a . 
\item[Prop5: ] $G_{\downarrow B^\prime}$ has only an empty admissible extension. 
\end{description}

%Formally, we have the following theorem. 

\begin{proposition}\label{theorem-pr}
Let $G^p = (A,R, p)$ be a PrAG, and $E\subseteq A$ be a conflict-free set of arguments. Then, for all $B\in 2^{I\cup E_G^+}$, $G_{\downarrow E\cup B}$ %be a complete subgraph of $G$ with respect to $E$. 
%Let $G^{\prime\prime} = (A^{\prime\prime}, R^{\prime\prime})$, where $A^{\prime\prime} = A^\prime \setminus (E\cup E^+)$ and $R^{\prime\prime} =R^\prime\cap(A^{\prime\prime}\times A^{\prime\prime})$. 
is a preferred subgraph of $G$ with respect to $E$ if and only if  $G_{\downarrow E\cup B}$ is a complete subgraph of $G$ with respect to $E$, and $\mathcal{E}_{ad}(G_{\downarrow B^\prime}) = \{\emptyset\}$.
\end{proposition}

\begin{proof}
$(\Rightarrow)$: Assume the contrary, i.e., $G_{\downarrow B^\prime}$ has a non-empty admissible extension $E^\prime\subseteq B^\prime$. It follows that $E\cup E^\prime$ is admissible, in that:
\begin{itemize}
\item $E\cup E^\prime$ is conflict-free: both $E$ and $E^\prime$ are conflict-free; $E$ does not attack $E^\prime$ (otherwise, $E^\prime\cap E_G^+\neq \emptyset$, contradicting $E^\prime\subseteq B^\prime$); $E^\prime$ does not attack $E$ (otherwise, $E$ attacks $E^\prime$, contradiction).
\item $\forall \alpha \in E^\prime$, $\alpha$ is acceptable with respect to $E\cup E^\prime$.
\end{itemize}
So,  $E\cup E^\prime$ is an admissible extension of $G_{\downarrow E\cup B}$. So, $E$ is not a preferred extension of $G_{\downarrow E\cup B}$, contradicting ``$G_{\downarrow E\cup B}$ is a preferred subgraph of $G$ with respect to $E$''.

$(\Leftarrow)$: Since $G_{\downarrow B^\prime}$ has only one empty admissible extension, no argument in $B^\prime$ is acceptable with respect to $E$ or any conflict-free superset of $E$. It turns out that $E$ is a preferred extension of $G_{\downarrow E\cup B}$, i.e., $G_{\downarrow E\cup B}$ is a preferred subgraph of $G$ with respect to $E$.
\end{proof}

According to Proposition \ref{theorem-pr}, the set of preferred subgraphs $\rho^{pr}(E)$ can be specified as follows:
\begin{equation} \label{formalus-prf}
\rho^{pr}(E) = \{G_{\downarrow E\cup B}\in \rho^{co}(E)  \mid \mathcal{E}_{ad}(G_{\downarrow B^\prime})= \{\emptyset\}\}
\end{equation}

\begin{example}
Continue Example \ref{ex-co}. Among the six complete subgraphs, except $G_1^1$ and $G_1^5$, others are preferred subgraphs with respect to $\{a\}$  (as shown in the sixth column of Table 2). 

With regard to ${G_1^1}$, $B^\prime = \{c,d\}$. Then, $ \mathcal{E}_{ad}(G_{\downarrow B^\prime})= \{\{c\}\} \neq \{\emptyset\}$. So, $G_1^1$ is not a preferred subgraph with respect to $\{a\}$. Similarly, $G_1^5$ is not a preferred subgraph with respect to $\{a\}$. 
\end{example}

Finally, given a complete subgraph $G_{\downarrow E\cup B}$, let $B^{\prime\prime} = B\cap E_G^+$. In order to verify whether it is a grounded subgraph, we may  simply check whether $G_{\downarrow E\cup B^{\prime\prime}}$ has a grounded extension $E$. To simplify the computation, we may divide $B^{\prime\prime}$ in to two disjoint subsets $B^{\prime\prime}_1$ and $B^{\prime\prime}_2$, where $B^{\prime\prime}_1 = B\cap (E_G^+\setminus E_G^-)$ and $B^{\prime\prime}_2 = B\cap (E_G^+\cap E_G^-)$, as illustrated in Figure 5. Note that arguments in $E_G^+\setminus E_G^-$ do not affect the the status of arguments in $E$. So, if the following property holds, then $G_{\downarrow E\cup B}$ is a grounded subgraph:

\begin{description}
%\item[Prop1: ] It is a . 
\item[Prop6: ] $G_{\downarrow E\cup B^{\prime\prime} _2}$ has a grounded extension $E$. 
\end{description}

\begin{figure}[!htp]
%  \begin{picture}(206,180)
\begin{center}
\includegraphics [width=0.8\textwidth]{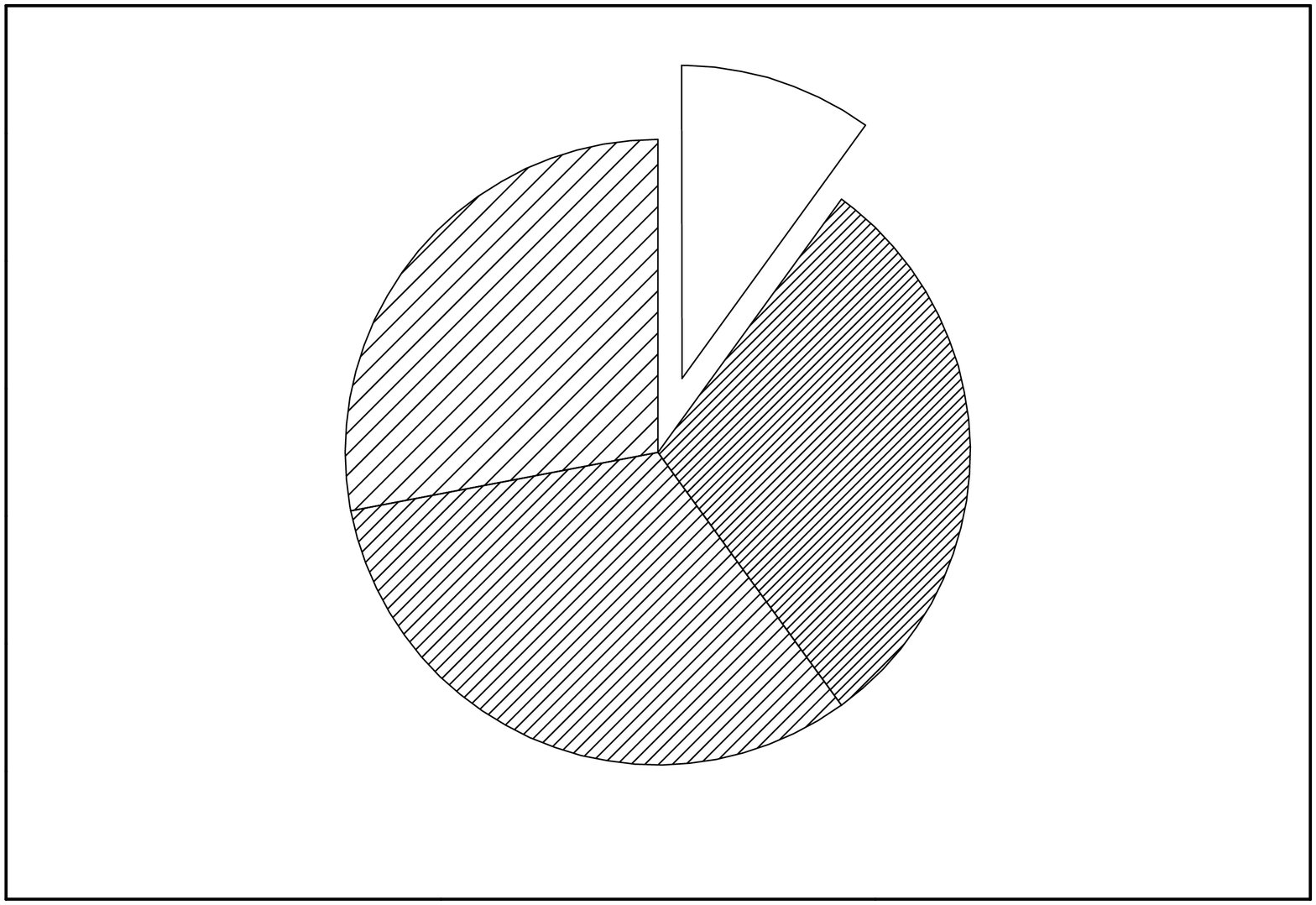}
 \put(-96,145){$E_G^+\setminus E_G^-$}
 \put(-90,85){$E_G^+\cap E_G^-$}
 \put(-120,187){$E_G^-\setminus E_G^+$}
 \put(-190,35){$E$}
 \put(-225,150){$I$}
\put(-185,130){\textcolor{grey}{\circle*{35}}}
\put(-190,126){$B^\prime$}
\put(-120,110){\textcolor{grey2}{\circle*{40}}}
 \put(-155,110){\line(6,-0){63}}
\put(-132,116){ $B^{\prime\prime}_1$}
\put(-132,98){ $B^{\prime\prime}_2$}
\end{center}
% \end{picture}
\caption{Given a PrAG $G = (A, R, p)$, a subgraph $G^\prime$ is a grounded subgraph w.r.t. $E$ if and only if it is a complete subgraph w.r.t. $E$, and $G_{\downarrow E\cup B^{\prime\prime} _2}$ has a grounded extension $E$.}
\end{figure}

\begin{proposition}\label{theorem-gr}
Let $G^p = (A,R, p)$ be a PrAG, $E\subseteq A$ be a conflict-free set of arguments. For all $B\in 2^{I\cup E_G^+}$, $G_{\downarrow E\cup B}$ is a grounded subgraph of $G$ with respect to $E$ if and only if $G_{\downarrow E\cup B}$ is a complete subgraph of $G$ with respect to $E$, and $E$ is a grounded extension of $G_{\downarrow E\cup B^{\prime\prime} _2}$.
%
%
%, and $G^\prime = (A^\prime, R^\prime)$ be a complete subgraph of $G^p$ with respect to $E$. Let $G^{\prime\prime} = (A^{\prime\prime}, R^{\prime\prime})$, where $A^{\prime\prime} = E\cup E^+$ and $R^{\prime\prime} =R^\prime\cap(A^{\prime\prime}\times A^{\prime\prime})$. Then,   ${G}^\prime$ is a grounded subgraph of $G^p$ with respect to $E$ if and only if the following condition holds: $E$ is a grounded extension of $G^{\prime\prime}$.
\end{proposition}

\begin{proof}
$(\Rightarrow)$: Since $G_{\downarrow E\cup B}$ is a grounded subgraph of $G$ with respect to $E$, it holds that $E$ is the grounded extension of $G_{\downarrow E\cup B}$. First, since a grounded extension is also a complete extension, $E$ is a complete extension of $G_{\downarrow E\cup B}$, i.e., $G_{\downarrow E\cup B}$ is a complete subgraph of $G$ with respect to $E$. Second, given that  $E$ is the grounded extension of $G_{\downarrow E\cup B}$, since $E$ does not receive any attacks from $B^\prime$ and $B_1^{\prime\prime}$, according to the directionality of grounded semantics \cite{Baroni:AIJ}, it holds that $E$ is the grounded extension of $G_{\downarrow E\cup B^{\prime\prime} _2}$, where $B^{\prime\prime} _2 = B\setminus (B^\prime\cup B_1^{\prime\prime})$. 

$(\Leftarrow)$: Since $G_{\downarrow E\cup B}$ is a complete subgraph of $G$ with respect to $E$, it holds that $E$ is a complete extension of $G_{\downarrow E\cup B}$. Now, we need to verify that $E$ is a minimal complete extension of $G_{\downarrow E\cup B}$.  Assume the contrary. There exists $E^\prime \subset E$ such that $E^\prime$ is a grounded extension of $G_{\downarrow E\cup B}$. According to the previous proof, it turns out that $E^\prime$ is a grounded extension of $G_{\downarrow E\cup B^{\prime\prime} _2}$, contradicting $E$ is a grounded extension of $G_{\downarrow E\cup B^{\prime\prime} _2}$.
\end{proof}

According to Proposition \ref{theorem-gr}, the set of grounded subgraphs $\rho^{gr}(E)$ can be specified as follows:
\begin{equation} \label{formalus-gr}
\rho^{gr}(E) = \{G_{\downarrow E\cup B}\in \rho^{co}(E)  \mid \mathcal{E}_{gr}(G_{\downarrow E\cup B^{\prime\prime} _2})= \{E\}\}
\end{equation}

\begin{example}
Continue Example \ref{ex-co}. Among the six complete subgraphs, $G_1^5$ and $G_1^7$ and  $G_1^8$ are grounded subgraphs with respect to $\{a\}$  (as shown in the last column of Table 2). 
\end{example}

\section{Semantics of probabilistic argumentation}
According to the theory introduced in the previous section, given a PrAG $G^p = (A,R, p)$, a conflict-free set of arguments $E\subseteq A$ and a semantics $\sigma\in\{ad, co, pr, gr, st\}$, a set of $\sigma$-subgraphs with respect to $E$ can be specified in terms of different properties related to $E$. Given \textbf{Prop1} - \textbf{Prop6} and  formulas (\ref{for-admiss}) - (\ref{formalus-gr}), we may define semantics of probabilistic argumentation by the following two approaches.

In the first place, according to formulas (\ref{for-admiss}) - (\ref{formalus-gr}) and (\ref{formula-2}), semantics of probabilistic argumentation, i.e., the probability of $E$ being a $\sigma$-extension (denoted as $p(E^\sigma)$), can be directly represented as follows.
\begin{eqnarray}
p(E^\sigma) &=& \Sigma_{G^\prime\in \rho^{\sigma}(E)}\,p(G^\prime) \label{formula-ns}
\end{eqnarray}

Note that $Q_\sigma(E)$ in formula (\ref{formula-2}) is replaced by $\rho^{\sigma}(E)$ in formula (\ref{formula-ns}).  

In this approach, although the characterized subgraphs are explicitly represented, they are not constructed blindly, but defined according to specific properties.  Therefore, the construction of most irrelevant subgraphs is avoided. Meanwhile, under admissible, complete and stable semantics, no computation of extensions is needed, while under preferred and grounded semantics, only the extension of the subgraphs induced by $B^\prime$ (resp. $E\cup B^{\prime\prime}$) is needed. Note that the number (resp. the average size) of  the subgraphs induced by $B^\prime$ (resp. $E\cup B^{\prime\prime}$) is usually much smaller than that of the subgraphs induced by $A^\prime\in 2^A$.  

In the second place,  a more efficient approach to define semantics of probabilistic argumentation is through directly using properties for characterizing subgraphs, such that the characterized subgraphs can be kept implicit  as much as possible. Now, let us introduce this approach under different semantics. 

First, under admissible semantics, according to \textbf{Prop1} and \textbf{Prop2},
%
%according to Theorem \ref{the-1}, each admissible subgraph with respect to an extension $E$ is characterized by the conditions under which all arguments in $E$ appear, while all arguments in $E^-\setminus E^+$ do not appear. In other words, the probability of a set of arguments being an admissible extension can be evaluated by the probabilities of arguments appearing or not appearing, without constructing the subgraphs and computing their extensions.
%
we have the following proposition.

\begin{proposition} \label{th-ad-sub}
Let $G^p = (A,R, p)$ be a PrAG, and $E\subseteq A$ be a conflict-free set of arguments. It holds that:
\begin{eqnarray}
p(E^{ad}) &=&  \Pi_{\alpha\in E}p(\alpha)\times\Pi_{\beta\in E_G^-\setminus E_G^+}p(\bar{\beta}) \label{formula-n-11c}  \nonumber
\end{eqnarray}
\end{proposition}

\begin{proof}
%Let $\Phi = A\setminus(E\cup (E^-\setminus E^+))$. 
According to formulas (\ref{for-admiss}) and (\ref{formula-ns}), $p(E^{ad}) = \Sigma_{G^\prime\in \rho^{ad}(E)}\,p(G^\prime) =  \Sigma_{B\in 2^{I\cup E_G^+}}\,p(G_{\downarrow E\cup B})$.  Since in $G_{\downarrow E\cup B}$,
\begin{itemize}
\item every argument in $E$ appears, 
\item every argument in $E_G^-\setminus E_G^+$ does not appear, 
\item every argument in $B$ appears, and
\item and every argument in $(I\cup E_G^+)\setminus B$ does not appear,
\end{itemize}
it holds that
$p(G_{\downarrow E\cup B})= \Pi_{\alpha\in E}p(\alpha)\times \Pi_{\beta\in E_G^-\setminus E_G^+}p(\bar{\beta})\times \Pi_{\gamma\in B}p(\gamma)\times \Pi_{\eta\in (I\cup E_G^+)\setminus B}p(\bar{\eta})$.
Since $ \Sigma_{B\in 2^{I\cup E_G^+}} (\Pi_{\gamma\in B}p(\gamma)\times \Pi_{\eta\in ({I\cup E_G^+})\setminus B}p(\bar{\eta})) =1$, we may conclude that:
\begin{eqnarray*}
p(E^{ad}) &=& \Sigma_{B\in 2^{I\cup E_G^+}}\,p(G_{\downarrow E\cup B})\\
&=& \Sigma_{B\in 2^{I\cup E_G^+}} (\Pi_{\alpha\in E}p(\alpha)\times \Pi_{\beta\in E_G^-\setminus E_G^+}p(\bar{\beta})\times\Pi_{\gamma\in B}p(\gamma)\times \Pi_{\eta\in ({I\cup E_G^+})\setminus B}p(\bar{\eta}))\\
&=& (\Pi_{\alpha\in E}p(\alpha)\times \Pi_{\beta\in E_G^-\setminus E_G^+}p(\bar{\beta})) \times \Sigma_{B\in 2^{I\cup E_G^+}} (\Pi_{\gamma\in B}p(\gamma)\times \Pi_{\eta\in ({I\cup E_G^+})\setminus B}p(\bar{\eta})) \\
&=&  \Pi_{\alpha\in E}p(\alpha)\times \Pi_{\beta\in E_G^-\setminus E_G^+}p(\bar{\beta})\times 1 \\
&=&  \Pi_{\alpha\in E}p(\alpha)\times \Pi_{\beta\in E_G^-\setminus E_G^+}p(\bar{\beta})
\end{eqnarray*}
\end{proof}
%
%\section{Empirical results}
%Average-case analysis requires a notion of an "average" input to an algorithm, which leads to the problem of devising a probability distribution over inputs. 
%
%The fundamental notions of average-case complexity were developed by Leonid Levin in 1986 when he published a one-page pape.
%
%An initial attempt might define an efficient average-case algorithm as one which runs in expected polynomial time over all possible inputs
%
%The common formalization of running almost always in Ptime is that for each $n$, the probability of hard instances of size $n$ is bounded by an inverse polynomial of $n$.
%
%For each $S\subseteq A$, under complete semantics, we need to check at most $2^{|A\setminus S\cup S^+|}$ subgraphs. So, theoretically, the computation of probability of an extension under complete semantics is NP complete. 
%
%However, the complexity depending on the ratio of $|S|: |A|$. This has not been studied before.  
%
%Meanwhile, it is useful to analyzed the complexity of identifying subgraphs.

%Besides, according to Theorem \ref{th-rel}, it holds that $p(E^{co})\in [p(E^{st}), p(E^{ad})]$ and $p(E^{pr})\in [p(E^{st}), p(E^{ad})]$. In other words, we may get $p(E^{co})$ and  $p(E^{pr})$, by
%
%\begin{itemize}
%\item directly using the definition of $\rho^{co}(E)$ and  $\rho^{pr}(E)$ respectively, or
%\item 
%\end{itemize}

Second, under stable semantics, according to \textbf{Prop1} and \textbf{Prop3}, we have the following proposition.

\begin{proposition} \label{th-st-sub}
Let $G^p = (A,R, p)$ be a PrAG, and $E\subseteq A$ be a conflict-free set of arguments. It holds that:
\begin{eqnarray}
p(E^{st}) &=&  \Pi_{\alpha\in E}p(\alpha)\times\Pi_{\beta\in I\cup (E_G^-\setminus E_G^+)}p(\bar{\beta}) \label{formula-n-11a} \nonumber
\end{eqnarray}
\end{proposition}

The proof of Proposition \ref{th-st-sub} is similar to that of Proposition \ref{th-ad-sub}, {so it is omitted.}

Third, under complete semantics, according to \textbf{Prop1}, \textbf{Prop2} and \textbf{Prop4}, we have the following proposition.

\begin{proposition} \label{th-co-sub}
Let $G^p = (A,R, p)$ be a PrAG, and $E\subseteq A$ be a conflict-free set of arguments. It holds that:
\begin{eqnarray}\label{formula-n-11b}
p(E^{co}) &=& P_E \times P_{I\_CO}, \mbox{where}  \nonumber\\
P_E &=&  \Pi_{\alpha\in E}p(\alpha)\times\Pi_{\beta\in E_G^-\setminus E_G^+}p(\bar{\beta}), \mbox{and} \nonumber \\
P_{I\_CO} &=&  \Sigma_{B^\prime\in 2^I\wedge (\forall \alpha\in B^\prime: \alpha_{G}^-\cap B^\prime\neq\emptyset)} (\Pi_{\gamma\in B^\prime}p(\gamma)\times \Pi_{\xi\in I\setminus B^\prime}p(\bar{\xi})) \nonumber 
\end{eqnarray}
\end{proposition}

\begin{proof}
According to formulas (\ref{formula-compl}) and (\ref{formula-ns}), $p(E^{co}) = \Sigma_{G^\prime\in \rho^{co}(E)}\,p(G^\prime) =  \Sigma_{(B\in 2^{I\cup E_G^+})\wedge  (\forall \alpha\in B^\prime: \alpha_{G}^-\cap B^\prime\neq\emptyset)}\,p(G_{\downarrow E\cup B})$.  Let $B^\prime =  B\cap I$  and $B^{\prime\prime} =  B\cap E_G^+$. It holds that $B = B^\prime\cup B^{\prime\prime} $ and $B^\prime\cap B^{\prime\prime}  = \emptyset$. Since in $G_{\downarrow E\cup B^\prime \cup B^{\prime\prime}}$,
\begin{itemize}
\item every argument in $E$ appears, 
\item every argument in $E_G^-\setminus E_G^+$ does not appear, 
\item every argument in $B^{\prime}$ (resp. $B^{\prime\prime}$) appears, 
\item and every argument in $I\setminus B^{\prime}$ (resp. $E_G^+\setminus B^{\prime\prime}$) does not appear,
\end{itemize}
it holds that $G_{\downarrow E\cup B} = G_{\downarrow E\cup B^\prime \cup B^{\prime\prime}} = (\Pi_{\alpha\in E}p(\alpha)\times \Pi_{\beta\in E_G^-\setminus E_G^+}p(\bar{\beta})\times\Pi_{\gamma\in B^\prime}p(\gamma)\times  \Pi_{\xi\in I\setminus B^\prime}p(\bar{\xi}) \times\Pi_{\zeta\in B^{\prime\prime}}p(\zeta)\times  \Pi_{\eta\in E_G^+\setminus B^{\prime\prime}}p(\bar{\xi}) )$. Since $ \Sigma_{B^{\prime\prime}\in 2^{E_G^+}} (\Pi_{\zeta\in B^{\prime\prime}}p(\zeta)\times  \Pi_{\eta\in E_G^+\setminus B^{\prime\prime}}p(\bar{\eta}) ) =1$, we may conclude that:
\begin{eqnarray*}
p(E^{co}) &=& \Sigma_{(B\in 2^{I\cup E_G^+})\wedge  (\forall \alpha\in B^\prime: \alpha_{G}^-\cap B^\prime\neq\emptyset)}\,p(G_{\downarrow E\cup B})\\
&=& \Sigma_{(B^\prime\in 2^{I})\wedge(B^{\prime\prime}\in 2^{E_G^+})\wedge  (\forall \alpha\in B^\prime: \alpha_{G}^-\cap B^\prime\neq\emptyset)}\,p(G_{\downarrow E\cup B^\prime\cup B^{\prime\prime}})\\
&=& \Sigma_{(B^\prime\in 2^{I})\wedge(B^{\prime\prime}\in 2^{E_G^+})\wedge  (\forall \alpha\in B^\prime: \alpha_{G}^-\cap B^\prime\neq\emptyset)}\,(\Pi_{\alpha\in E}p(\alpha)\times \Pi_{\beta\in E_G^-\setminus E_G^+}p(\bar{\beta})\times \\
&&\Pi_{\gamma\in B^\prime}p(\gamma)\times  \Pi_{\xi\in I\setminus B^\prime}p(\bar{\xi}) \times\Pi_{\zeta\in B^{\prime\prime}}p(\zeta)\times  \Pi_{\eta\in E_G^+\setminus B^{\prime\prime}}p(\bar{\xi}) )\\
&=& (\Pi_{\alpha\in E}p(\alpha)\times \Pi_{\beta\in E_G^-\setminus E_G^+}p(\bar{\beta})) \times \Sigma_{B^{\prime\prime}\in 2^{E_G^+}}(\Pi_{\zeta\in B^{\prime\prime}}p(\zeta)\times  \Pi_{\eta\in E_G^+\setminus B^{\prime\prime}}p(\bar{\eta}) )\times \\
&&  \Sigma_{(B^\prime\in 2^{I})\wedge  (\forall \alpha\in B^\prime: \alpha_{G}^-\cap B^\prime\neq\emptyset)}(\Pi_{\gamma\in B^\prime}p(\gamma)\times  \Pi_{\xi\in I\setminus B^\prime}p(\bar{\xi}) ) \\
&=& P_E\times 1 \times P_{I\_CO} \\
&= & P_E\times  P_{I\_CO}
\end{eqnarray*}
\end{proof}

Third, under preferred semantics, according to \textbf{Prop1}, \textbf{Prop2},  \textbf{Prop4} and \textbf{Prop5}, we have the following proposition.

\begin{proposition} \label{th-pr-sub}
Let $G^p = (A,R, p)$ be a PrAG, and $E\subseteq A$ be a conflict-free set of arguments. It holds that:
\begin{eqnarray}\label{formula-n-11b}
p(E^{pr}) &=& P_E \times P_{I\_PR}, \mbox{where}  \nonumber\\
P_E &=&  \Pi_{\alpha\in E}p(\alpha)\times\Pi_{\beta\in E_G^-\setminus E_G^+}p(\bar{\beta}), \mbox{and} \nonumber \\
P_{I\_PR} &=&  \Sigma_{(B^\prime\in 2^I)\wedge (\forall \alpha\in B^\prime: \alpha_{G}^-\cap B^\prime\neq\emptyset)\wedge (\mathcal{E}_{ad}(G_{\downarrow B^\prime}) = \{\emptyset\})} (\Pi_{\gamma\in B^\prime}p(\gamma)\times \Pi_{\xi\in I\setminus B^\prime}p(\bar{\xi})) \nonumber 
\end{eqnarray}
\end{proposition}

The proof of Proposition \ref{th-pr-sub} is similar to that of Proposition \ref{th-co-sub}, omitted.

Third, under grounded semantics, according to \textbf{Prop1}, \textbf{Prop2},  \textbf{Prop4} and \textbf{Prop6}, we have the following proposition.

\begin{proposition} \label{th-gr-sub}
Let $G^p = (A,R, p)$ be a PrAG, and $E\subseteq A$ be a conflict-free set of arguments. It holds that:
\begin{eqnarray}\label{formula-n-11e}
p(E^{gr}) &=& p(E^{co}) \times P_{GR}, \mbox{where}  \nonumber \\
%P_E &=&  \Pi_{\alpha\in E}p(\alpha)\times\Pi_{\beta\in E_G^-\setminus E_G^+}p(\bar{\beta}), \mbox{and} \nonumber \\
P_{GR} &=&  \Sigma_{(B^{\prime\prime}_2\in 2^{E_G^+\cap E_G^-})\wedge (\mathcal{E}_{gr}(G_{\downarrow E\cup B^{\prime\prime}}) = \{\{E\}\})} (\Pi_{\alpha\in B^{\prime\prime}}p(\alpha)\times \Pi_{\beta\in (E_G^+\cap E_G^-)\setminus B^{\prime\prime}}p(\bar{\beta})) \nonumber 
\end{eqnarray}
\end{proposition}

The proof of Proposition \ref{th-gr-sub} is similar to that of Proposition \ref{th-co-sub}. The difference is that $B^{\prime\prime}$ is divided into two parts: $B^{\prime\prime}_1$ and $B^{\prime\prime}_2$, in which the arguments in $B^{\prime\prime}_1$ appear without constraints while the arguments in $B^{\prime\prime}_2$ appear only when $\mathcal{E}_{gr}(G_{\downarrow E\cup B^{\prime\prime}}) = \{\{E\}\}$ is satisfied. This is reflected by the factor $P_{GR}$ and an equation $ \Sigma_{B^{\prime\prime}_1\in 2^{E_G^+\setminus E_G^-}} (\Pi_{\alpha\in B^{\prime\prime}_1}p(\alpha)\times  \Pi_{\beta\in (E_G^+\setminus E_G^-)\setminus B^{\prime\prime}_1}p(\bar{\beta}) ) =1$.

\section{Algorithms and empirical results}
The theoretical results presented in the previous section show that our characterized subgraphs based approach (called \textit{C-Sub approach}) could be more efficient than the possible worlds based approach (called \textit{PW approach}). In order to quantitatively evaluate the performance of our approach, in this section, by taking the cases under preferred semantics as an example,  we first develop two algorithms for the PW approach  and the C-Sub approach under preferred semantics respectively, and then conduct experiments to obtain the empirical results\footnote{The reasons why we choose  preferred semantics for our empirical study here are as follows. First,  efficiency of our approaches are mainly affected by the size of the set of remaining arguments, which is mainly dependent on the structure of graphs and the size of the extension, rather than on the semantics we choose.  Second, preferred semantics has been widely used in many experiments (e.g., \cite{Wolfgang:AIJ2012B}, \cite{Cerutti} and \cite{Nofal}, among others).  And, we have also conducted several experiments under preferred semantics \cite{Liao-Huang:JLC, Liao:amai}. Third, for some other properties, we will study them in our future work. For instance, under grounded semantics, an additonal property is how the efficiency of the new approach is affected by the size of $E_G^+\cap E_G^-$.}.

\subsection{Algorithm for the PW approach under preferred semantics}

Alg. 1 is an algorithm for the PW approach under preferred semantics. In this algorithm, up to $2^n$ subgraphs are blindly constructed where $n$ is the number of nodes of the PrAG. For each subgraph $G_{\downarrow A^\prime}$, if $E\subseteq A^\prime$, then whether $E$ is one of its preferred extensions is verified by the procedure $\mathit{verify\_preferred\_labelling}(\mathcal{L}, E)$. This procedure is based on the  MC algorithm introduced in Section 2.  

$\mathcal{L}$ is initialized as $(A^\prime, \emptyset, \emptyset)$, i.e. all arguments in $A^\prime$ are all labelled IN. Then, the procedure first checks whether there is an argument  $\alpha$ in $E$ such that $\alpha$ is super-illegally IN with respect to $\mathcal{L}$. If so, $E$ is not a preferred extension. Otherwise, there are two possible cases. First, no argument is illegally IN. It follows that $in(\mathcal{L})$ is admissible. In this case, if $E\subset in(\mathcal{L})$, then $E$ is not a  preferred extension. Second, there are some arguments that are illegally IN. In this case, the procedure iteratively  selects arguments that are illegally IN (or \textit{super-illegally IN}) and applies a \textit{transition step} to obtain a new labelling, until a labelling is reached in which no argument is illegally IN. %In this procedure, we use the notions of \textit{super-illegally IN} and \textit{transition step}, which are introduced as follows.

\begin{figure}[!htp]
\begin{center}
\setlength{\parindent}{0cm}
\begin{tabular}{l}
\hline
\textbf{Alg. 1}\hspace{0.05cm}  \textbf{Algorithm for the PW approach under preferred semantics}\\
%\hspace{2.8cm}\textbf{$(A, R)$} \\
\hline
\small \hspace{0.1cm}   $ \,\,\,$ \textbf{input    :}  $G^p = (A, R, p)$ and a conflict-free set $E\subseteq A$ \\
%
%$\mathcal{L} = (in(\mathcal{L}), out(\mathcal{L}), undec(\mathcal{L}))$, $E \subseteq A$\\
\small\hspace{0.1cm}  $ \,\,\,$ \textbf{output:}  $p(E^{pr})$\\
\small\hspace{0.1cm} 1:  $p(E^{pr}): =0;$ \\
\small\hspace{0.1cm} 2:  for each $A^\prime \in 2^A$ and $A^\prime \neq \emptyset$ do   \\
\small\hspace{0.1cm} 3: \hspace{0.5cm} if $E\subseteq A^\prime$  then \\
\small\hspace{0.1cm} 4: \hspace{0.9cm} $\mathcal{L} := (A^\prime, \emptyset, \emptyset) $; \\
\small\hspace{0.1cm} 5: \hspace{0.9cm} if $\mathit{verify\_preferred\_labelling}(\mathcal{L}, E) = true$  then \\
\small\hspace{0.1cm} 6: \hspace{1.3cm}  $p(E^{pr}) := p(E^{pr}) + p(G_{\downarrow A^\prime});$  \\
\small\hspace{0.1cm} 7: \hspace{0.9cm} end if  \\
\small\hspace{0.1cm} 8: \hspace{0.5cm} end if  \\
\small\hspace{0.1cm} 9:  end do\\
\\
\small\hspace{0.1cm} 10: \textbf{procedure} $\mathit{verify\_preferred\_labelling}(\mathcal{L}, E$) \\
\small\hspace{0.1cm} 11:  $vpr := true$;\\
\small\hspace{0.1cm} 12: \textbf{if} $\mathcal{L}$ has an argument $\alpha$ that is super-illegally $\mathrm{IN}$ and  $\alpha\in E$ \textbf{then} \\
\small\hspace{0.1cm} 13: \hspace{0.5cm} $vpr:=false$; \\
%\small\hspace{0.1cm} 14: \textbf{end if}\\ 
\small\hspace{0.1cm} 15: \textbf{else}\\ 
\small\hspace{0.1cm} 16: \hspace{0.5cm} \textbf{if} $\mathcal{L}$ does not have an argument that is illegally $\mathrm{IN}$ \textbf{then}\\ \small\hspace{0.1cm} 17: \hspace{0.5cm} \hspace{0.5cm}\textbf{if }$E\subset in(\mathcal{L})$ \textbf{then} $vpr:=false$; \textbf{end if}\\
\small \hspace{0.1cm} 18: \hspace{0.5cm} \textbf{else} \\
\small \hspace{0.1cm} 19: \hspace{0.9cm}  \textbf{if} $\mathcal{L}$ has an argument that is super-illegally $\mathrm{IN}$ \textbf{then} \\
\small \hspace{0.1cm} 20: \hspace{1.3cm} $\alpha := $ some argument that is super-illegally $\mathrm{IN}$ in $\mathcal{L}$;  \\
\small \hspace{0.1cm} 21: \hspace{1.3cm} $\mathit{verify\_preferred\_labelling}(\mathit{transition\_step}(\mathcal{L}, \alpha), E$); \\
\small \hspace{0.1cm} 22: \hspace{0.9cm}  \textbf{else}  \\
\small \hspace{0.1cm} 23:  \hspace{1.3cm}  \textbf{for each} $\alpha$ that is illegally $\mathrm{IN}$ in $\mathcal{L}$ \textbf{do} \\
\small \hspace{0.1cm} 24: \hspace{1.7cm} $\mathit{verify\_preferred\_labelling}(\mathit{transition\_step}(\mathcal{L}, \alpha), E$)\\
\small\hspace{0.1cm} 25:  \hspace{1.3cm}  \textbf{end for} \\
\small \hspace{0.1cm} 26: \hspace{0.9cm} \textbf{end if} \\
\small  \hspace{0.1cm} 27: \hspace{0.5cm} \textbf{end if} \\
\small  \hspace{0.1cm} 28: \textbf{end if} \\
\small  \hspace{0.1cm} 29: \textbf{return} $vpr$\\
\small  \hspace{0.1cm} 30: \textbf{end procedure}\\
\hline
\end{tabular}
\setlength{\parindent}{0.5cm}
\end{center}
\end{figure}

\subsection{Algorithm for the C-Sub approach under preferred semantics}
Alg. 2 is an algorithm for the C-Sub approach under preferred semantics. Unlike the PW approach, the algorithm first gets a set of remaining argument $I = A\setminus (E^-\cup E^+\cup E)$. Then, for each subset $B^\prime$ of $I$, verify whether the subgraph induced by $B^\prime$ has an nonempty admissible extension. The procedure $\mathit{verify\_nonempty\_adm}((\mathcal{L})$ recursively  selects arguments that are illegally IN (or \textit{super-illegally IN}) and applies a \textit{transition step} to obtain a new labelling, until a lablling is reached in which no argument is illegally IN. If there is a labeling $\mathcal{L}$ such that $\mathcal{L}$ has no argument that is  illegally $\mathrm{IN}$ and $in(\mathcal{L})\neq \emptyset$, then the procedure returns true. Otherwise, it returns false. 

Then, $p(E^{pr})$ is computed according to Proposition \ref{th-pr-sub}. More specifically, in Steps 11 and 12, $P_E$ is computed; from Step 13 to Step 17, $P_{I\_PR}$ and $p(E^{pr})$ are computed. 

%\subsection{Our algorithm}
\begin{figure}[]
\begin{center}
\setlength{\parindent}{0cm}
\begin{tabular}{l}
\hline
\textbf{Alg. 2}\hspace{0.05cm}  \textbf{Algorithm for C-Sub approach under preferred semantics}\\
%\hspace{2.8cm}\textbf{$(A, R)$} \\
\hline
\small$ $ \hspace{0.2cm}   $ \,\,\,$ \textbf{input    :}  $G^p = (A, R, p)$, a conflict-free set $E\subseteq A$\\
\small\hspace{0.3cm}  $ \,\,\,$ \textbf{output:}  $p(E^{pr})$\\
\small\hspace{0.1cm} 1:  $p(E^{pr}): =0;$ $sub : = \emptyset$; $I := A\setminus (E^-\cup E^+\cup E)$; $x_1 : = 1;$  $x_2 : = 1;$\\
\small\hspace{0.1cm} 2: \textbf{for} each $B^\prime\in 2^{I}$ \textbf{do}   \\
\small\hspace{0.1cm} 3: \hspace{0.4cm} \textbf{if} $B^\prime = \emptyset$ \textbf{then} $sub := sub \cup \{\emptyset\}$;\\
\small\hspace{0.1cm} 4: \hspace{0.4cm} \textbf{else} \\
\small\hspace{0.1cm} 5:  \hspace{0.8cm} \textbf{if} there exists no $\alpha\in B^\prime$ such that $\alpha_G^-\cap B^\prime = \emptyset$ \textbf{then} \\
\small\hspace{0.1cm} 6: \hspace{1.2cm} $\mathcal{L} := (B^\prime, \emptyset, \emptyset)$; \\
\small\hspace{0.1cm} 7: \hspace{1.2cm}  \textbf{if} $\mathit{verify\_nonempty\_adm}(\mathcal{L}) = false$ \textbf{then} $sub := sub \cup \{B^\prime\}$;  \textbf{end if}  \\
\small\hspace{0.1cm} 8: \hspace{0.8cm}   \textbf{end if} \\
\small\hspace{0.1cm}  9: \hspace{0.4cm}   \textbf{end if} \\
\small 10:  \textbf{end for} \\
\small 11: \textbf{for} each $\alpha\in E_G^-\setminus E_G^+$ \textbf{do}  $x_1 : = x_1 \times p(\bar{\alpha});$  \textbf{end for} \\
\small 12: \textbf{for} each $\alpha\in E$ \textbf{do}  $x_1 : = x_1 \times p({\alpha});$   \textbf{end for} \\

\small 13: \textbf{for} each $B^\prime\in sub$ \textbf{do} \\
\small 14: \hspace{0.4cm}  \textbf{for} each $\alpha\in I\setminus B^\prime$ \textbf{do} $x_2 : = x_2 \times p(\bar{\alpha});$   \textbf{end for} \\
\small 15: \hspace{0.4cm}  \textbf{for} each $\alpha\in B^\prime$ \textbf{do} $x_2 : = x_2 \times p({\alpha});$   \textbf{end for} \\
\small 16:  \hspace{0.4cm}  $p(E^{pr}): =p(E^{pr}) + x_1\times x_2;$  \\
\small 17: \textbf{end for}  \\

\\
\small 18: \textbf{procedure} $\mathit{verify\_nonempty\_adm}(\mathcal{L})$ \\
\small 19:  $vna := false$;\\
\small 20:  \textbf{if} $\mathcal{L}$ has no argument that is illegally $\mathrm{IN}$ \textbf{then} \\
\small 21:   \hspace{0.4cm} if $in(\mathcal{L}) \neq \emptyset$ then $vna := true$;\\
\small 22:  \textbf{else} \\
\small 23:  \hspace{0.4cm} \textbf{if} $\mathcal{L}$ has an argument that is super-illegally $\mathrm{IN}$ \textbf{then} \\
\small 24: \hspace{0.8cm} $\alpha := $ some argument that is super-illegally $\mathrm{IN}$ in $\mathcal{L}$;  \\
\small 25: \hspace{0.8cm} $\mathit{verify\_nonempty\_adm}(\mathit{transition\_step}(\mathcal{L}, \alpha)$); \\
\small 26: \hspace{0.4cm} \textbf{else}  \\
\small 27:  \hspace{0.8cm}  \textbf{for each} $\alpha$ that is illegally $\mathrm{IN}$ in $\mathcal{L}$ \textbf{do} \\
\small 28: \hspace{1.2cm} $\mathit{verify\_nonempty\_adm}(\mathit{transition\_step}(\mathcal{L}, \alpha)$);\\
\small 29:  \hspace{0.8cm}  \textbf{end for} \\
\small 30:  \hspace{0.4cm} \textbf{end if} \\
\small 31:  \textbf{end if} \\
\small 32: \textbf{return} $vna$;\\
\small 33: \textbf{end procedure}\\

\hline
\end{tabular}
\setlength{\parindent}{0.5cm}
\end{center}
\end{figure} 

\subsection{Empirical results}
%
%Under preferred semantics, given a PrAG and a set of conflict-free arguments $S$, we use the existing possible world based approach (or briefly PW approach) and the characterized subgraphs bases approach (or briefly C-Sub approach)  respectively to compute the probability of $S$ being a preferred extension of the PrAG. 
%
%For the PW approach, given a PrAG, $2^n-1$ subgraphs are constructed, which $n$ is the number of the nodes of the graph.   For each sub-graph, we verify whether $S$ is a preferred extension of this subgraph. The algorithm for verifying a set $S$ is a preferred extension is as follows. 
%
%
%On the other hand, in our approach, 

The algorithms were implemented in Java, and tested on a machine with an Intel CPU running at 2.26 GHz and 2.00 GB RAM. We conducted three experiments to test the performance of out C-Sub approach.

The first experiment is about the average computation time of the C-Sub approach and that of the PW approach, according to the following configuration of PrAGs:
\begin{itemize}
\item The number of nodes of PrAGs is from 10 to 25 (since when the number of nodes is smaller than 10, the computation time of the two approaches is close to 0 millisecond, while the number of nodes is bigger than 25, the computation time of the PW approach is almost always more than 3 minutes which we set as the point of timeout). 
\item The ratios of the number of edges to the number of nodes are 1:1, 2:1 and 3:1 respectively (since the density of PrAGs is an important factor affecting the average computation time of the two approaches). 
\item The size of the extension is 3. This number is selected somewhat arbitrarily. How the size of the extension affects the average computation time of the two approaches will be studied in another experiment. 
\end{itemize}

This configuration consists of $16\times 3\times 1 = 48$ assignments for the two approaches respectively. Each assignment is a tuple $(\mbox{\#nodes}, i:1, j)$, where ``\#nodes'' is  the number of nodes, $i:1$ is the ratio of the number of edges to the number of nodes, and $j$ is the size of the extension. For convenience, we use PW\_$j [i:1]$ (C-Sub\_$j [i:1]$)  to denote the (average) computation time of the PW approach (resp. the C-Sub approach) when the size of the extension is $j$ and the the ratio of the number of edges to the number of nodes is $i:1$,  and the number of nodes is given.  

For each assignment, the algorithms were executed 20 times respectively. In each time, a PrAG (including its notes, edges, and the probabilities of nodes) and a conflict-free set $E$ of arguments were generated at random. For simplicity, the probabilities assigned to nodes are nonzero. Then, the probability of $E$ being a preferred extension was computed by the PW approach and the C-Sub approach respectively. Table 3 shows the average execution time of the two approaches. 

\begin{table}[!htp]
\begin{center}
\begin{tabular}{|c|c|c|c||c|c|c|}
  \hline
\scriptsize \#nodes &\scriptsize PW\_3 [1:1] & \scriptsize PW\_3 [2:1] &\scriptsize PW\_3 [3:1]&\scriptsize C-Sub\_3 [1:1] &\scriptsize C-Sub\_3 [2:1]  &\scriptsize C-Sub\_3 [3:1] \\
    & \footnotesize \parbox{1.35cm}{(secs/ timeout)} & \footnotesize \parbox{1.35cm}{(secs/ timeout)}  &\footnotesize \parbox{1.35cm}{(secs/ timeout)} &\footnotesize\parbox{1.35cm}{(secs/ timeout)} & \footnotesize\parbox{1.35cm}{(secs/ timeout)} &\footnotesize\parbox{1.35cm}{(secs/ timeout)}  \\
  \hline
  \hline
\footnotesize 10 & \footnotesize 0.015/0& \footnotesize  0.120/0& \footnotesize  0.585/0& \footnotesize 0.001/0& \footnotesize 0.000/0& \footnotesize  0.000/0 \\

  \footnotesize 11 & \footnotesize 0.039/0& \footnotesize 0.648/0& \footnotesize  4.346/0& \footnotesize 0.000/0& \footnotesize 0.000/0& \footnotesize  0.000/0\\

  \footnotesize 12& \footnotesize 0.070/0& \footnotesize 3.141/0& \footnotesize  34.662/0 & \footnotesize 0.002/0  & \footnotesize 0.001/0& \footnotesize  0.000/0\\
  \footnotesize 13& \footnotesize 0.160/0& \footnotesize 6.732/0& \footnotesize   \underline{101.667}/6 & \footnotesize 0.006/0& \footnotesize 0.005/0& \footnotesize  0.000/0\\

  \footnotesize 14 & \footnotesize 0.380/0 & \footnotesize \underline{23.879}/1& \footnotesize   \underline{154.271}/13& \footnotesize 0.000/0& \footnotesize 0.000/0& \footnotesize  0.001/0\\

  \footnotesize 15 & \footnotesize 4.772/0 & \footnotesize \underline{87.185}/8& \footnotesize   /20 & \footnotesize 0.003/0& \footnotesize 0.002/0& \footnotesize  0.002/0\\
  \footnotesize 16& \footnotesize 2.236/0 & \footnotesize \underline{112.569}/11& \footnotesize   /20& \footnotesize  0.001/0& \footnotesize 0.003/0 & \footnotesize  0.000/0\\
  \footnotesize 17& \footnotesize 11.674/0& \footnotesize \underline{107.201}/9& \footnotesize  /20& \footnotesize 0.003/0& \footnotesize 0.003/0& \footnotesize  0.001/0\\
  \footnotesize 18& \footnotesize \underline{18.445}/1 & \footnotesize \underline{149.583}/16& \footnotesize  /20& \footnotesize 0.012/0& \footnotesize 0.003/0& \footnotesize 0.019/0\\
  \footnotesize 19 & \footnotesize  \underline{31.282}/1 & \footnotesize \underline{159.580}/17& \footnotesize  /20& \footnotesize 0.015/0& \footnotesize 0.011/0& \footnotesize 0.023/0\\
  \footnotesize 20& \footnotesize  \underline{50.973}/2& \footnotesize  /20& \footnotesize  /20& \footnotesize 0.028/0& \footnotesize 0.477/0& \footnotesize  0.211/0 \\
 \footnotesize 21&\footnotesize  \underline{89.654}/5& \footnotesize/20& \footnotesize  /20& \footnotesize 0.088/0 & \footnotesize 0.067/0& \footnotesize  \underline{11.665}/1 \\
  \footnotesize 22&\footnotesize  \underline{143.039}/10&\footnotesize /20& \footnotesize  /20& \footnotesize 0.106/0 & \footnotesize  \underline{14.925}/1& \footnotesize  \underline{10.043}/1\\
  \footnotesize 23&\footnotesize /20& \footnotesize /20& \footnotesize  /20   & \footnotesize 0.627/0& \footnotesize  {6.901}/0& \footnotesize  \underline{13.825}/1\\
  \footnotesize 24&\footnotesize /20& \footnotesize /20  & \footnotesize  /20 & \footnotesize 3.067/0& \footnotesize \underline{12.434}/1& \footnotesize  \underline{0.741}/0\\

  \footnotesize 25&\footnotesize /20& \footnotesize /20 & \footnotesize  /20  & \footnotesize {1.406}/0 & \footnotesize \underline{30.222}/3& \footnotesize  \underline{15.429}/1\\

  \hline
\end{tabular}
\caption{ The average execution time of the PW approach and the C-Sub approach}
\end{center}
\end{table}

Since in many cases,  the execution time might last very long, to make the test possible, when the time for computing $p(E^{pr})$ is over 3 minutes (180 seconds), the execution was stopped by setting a break in the program.  When the number of timeout is less than 20, the average time was recorded, and for each timeout, the time used for calculation is 180 seconds. For instance, when \#nodes = 25,  C-Sub\_3 [3] = 15.828 seconds. The detailed records of 20 times of execution are shown in Table 4. For instance, C-Sub\_3 [3] = (0.016 + 3.760 + 0.000 + 0.015 + 22.074 + 0.016 + 24.039 + 0.078 + 180 + 4.524 + 43.275 + 0.000 + 0.000 + 0.016 + 0.000 + 0.015 + 0.000 + 30.732 + 0.016 + 0.000) $\div 20 =15.429$.

\begin{table}[!tp]
\begin{center}
\begin{tabular}{|r|c|c|c||c|c|c|}
  \hline
 \footnotesize No. &\multicolumn{3}{c}{\footnotesize {C-Sub\_3 [3:1]} (\#nodes = 25)}\vline&\multicolumn{3}{c}{\footnotesize {C-Sub\_3 [2:1]} (\#nodes = 25)} \vline \\
 \cline{2-7}
    & \footnotesize time (secs)  & \footnotesize {max $|B^\prime|$} & \footnotesize {avg. $|B^\prime|$}&   \footnotesize time (secs)  & \footnotesize {max $|B^\prime|$} & \footnotesize {avg. $|B^\prime|$}  \\
  \hline
  \hline
  \footnotesize 1&   \footnotesize 0.016&   \footnotesize 8 &   \footnotesize 4 &   \footnotesize \textbf{timeout}&   \footnotesize \textbf{15} &   \footnotesize \textbf{1}    \\

  \footnotesize    \footnotesize 2&   \footnotesize {3.760}&   \footnotesize {12} &   \footnotesize {6 } &   \footnotesize 0.640&   \footnotesize 15 &  \footnotesize  7\\

   \footnotesize 3&   \footnotesize 0.000&   \footnotesize 10 &   \footnotesize 5 &   \footnotesize 0.000&   \footnotesize 10 &  \footnotesize  5  \\

   \footnotesize 4&   \footnotesize 0.015&  \footnotesize  7 &   \footnotesize 3 &   \footnotesize 0.281&   \footnotesize 16 &   \footnotesize 8   \\

   \footnotesize \textbf{5}&   \footnotesize \textbf{22.074}&  \footnotesize  \textbf{11} &  \footnotesize  \textbf{5} &  \footnotesize  0.000&  \footnotesize  11 &  \footnotesize  5   \\

   \footnotesize 6&   \footnotesize 0.016&  \footnotesize  10&  \footnotesize  5 &  \footnotesize  4.339&  \footnotesize  13 &  \footnotesize  6   \\

  \footnotesize \textbf{ 7}&  \footnotesize  \textbf{24.039}&  \footnotesize  \textbf{11 }&  \footnotesize \textbf{ 5  }  &  \footnotesize  \textbf{57.424}&  \footnotesize  \textbf{14} &  \footnotesize  \textbf{7}\\

  \footnotesize  8&  \footnotesize  0.078&  \footnotesize  11&  \footnotesize  5  &  \footnotesize  0.078&  \footnotesize  14&  \footnotesize  7  \\

  \footnotesize \textbf{ 9}&  \footnotesize  \textbf{timeout }&   \footnotesize \textbf{13} &  \footnotesize \textbf{5}  &  \footnotesize  \textbf{timeout} &  \footnotesize  \textbf{16 }&  \footnotesize  \textbf{1}  \\

   \footnotesize 10&  \footnotesize  4.524&  \footnotesize  11 &  \footnotesize  5 &  \footnotesize  0.000&  \footnotesize 10&  \footnotesize  5   \\

  \footnotesize 11&  \footnotesize \textbf{43.275}&  \footnotesize  \textbf{14}&  \footnotesize  \textbf{7}  &   \footnotesize 0.031&  \footnotesize  13 &  \footnotesize  6 \\

   \footnotesize 12&   \footnotesize 0.000 &  \footnotesize  9&   \footnotesize 4  &  \footnotesize  0.047&  \footnotesize  14 &   \footnotesize 7  \\

   \footnotesize 13&   \footnotesize 0.000&   \footnotesize 9 &  \footnotesize  4  &  \footnotesize  0.125&  \footnotesize  15 &   \footnotesize 7  \\

  \footnotesize 14&  \footnotesize  0.016&   \footnotesize 9 &   \footnotesize 4  &   \footnotesize 0.047&  \footnotesize  14 &   \footnotesize 7  \\

   \footnotesize 15&  \footnotesize 0.000&  \footnotesize  11 &  \footnotesize  5 &   \footnotesize 1.310&  \footnotesize  13 &  \footnotesize  6   \\

   \footnotesize 16&  \footnotesize  0.015&   \footnotesize 11 &   \footnotesize 5  &   \footnotesize\textbf{ timeout}&  \footnotesize  \textbf{16 }&  \footnotesize  \textbf{3} \\

   \footnotesize 17&  \footnotesize 0.000&   \footnotesize 9 &   \footnotesize 4  &   \footnotesize 0.047&  \footnotesize  13 &  \footnotesize  6  \\

   \footnotesize \textbf{18}&   \footnotesize \textbf{30.732}&   \footnotesize \textbf{12} &  \footnotesize  \textbf{6}  &  \footnotesize  0.031&  \footnotesize  12 &   \footnotesize 6 \\

  \footnotesize 19&   \footnotesize {0.016}&  \footnotesize  {9} &  \footnotesize  {4}  &  \footnotesize  0.016&  \footnotesize  12 &  \footnotesize  6 \\

   \footnotesize 20&  \footnotesize  0.000&  \footnotesize  10 &  \footnotesize  5&   \footnotesize 0.062&  \footnotesize  11 &   \footnotesize 5   \\
 \hline
avg. &\footnotesize 15.429&  \footnotesize  10.35 &   \footnotesize  &   \footnotesize 30.222 &   \footnotesize 13.35 & \\
 \hline
\end{tabular}
\end{center}
\caption{The detailed records of the execution time of the C-Sub approach. In this table, ``max $|B^\prime|$'' and ``avg. $|B^\prime|$'' denote respectiely the maximal and average size of $B^\prime$.}
\end{table}

From Table 3, we found that the C-Sub approach greatly outperforms the PW approach. The computation time of the PW approach increases dramatically with the increase of the number of nodes and the density of edges. More specifically, when the number of nodes is given, PW\_3 [$i:1$] increases sharply with the increase of $i$. For instance, when \#nodes = 15, PW\_3 [$1:1$] = 4.772, PW\_3 [$2:1$] = 87.185 (with 8 timeouts), and PW\_3 [$3:1$] has no record of time (with 20 timeouts). Meanwhile, when the density of edges is given, PW\_3 [$i:1$] ($i = 1, 2, 3$) increases exponentially with the increase of \#nodes. On the contrary, with the increase of density (i.e., $i:1$), C-Sub\_3 [$i:1$] might not increase. And, with the increase of the number of nodes, PW\_3 [$i:1$] ($i = 1, 2, 3$) does not increase exponentially. The basic reason behind these phenomena is that according to the theoretical results obtained in Section 4, compared to the PW approach, the complexity of the C-Sub approach decreases from $|2^A|$ to $|2^I|$. In other words, the complexity of the C-Sub approach is manly determined by the size of $I$ (i.e., the maximal size of $B^\prime$). This is evidenced by the data shown in Table 4, in which the average value of maximal sizes of $B^\prime$ in 20 tests is 10.35 for C-Sub\_3 [$3:1$] and 13.35 for C-Sub\_3 [$2:1$], which matches very well to the average computation time of C-Sub\_3 [$3:1$] (15.429 seconds) and C-Sub\_3 [$2:1$] (30.222 seconds).
%More specifically, the computation time of the PW approach increases exponentially with the increase of the number of nodes, while the computation time of the C-Sub approach is much smaller. 
\begin{figure}[!hp]
\begin{center}
\begin{tabular}{ll}
%\vspace{-1.5cm}(a)& \\
\includegraphics [width=0.515\textwidth]{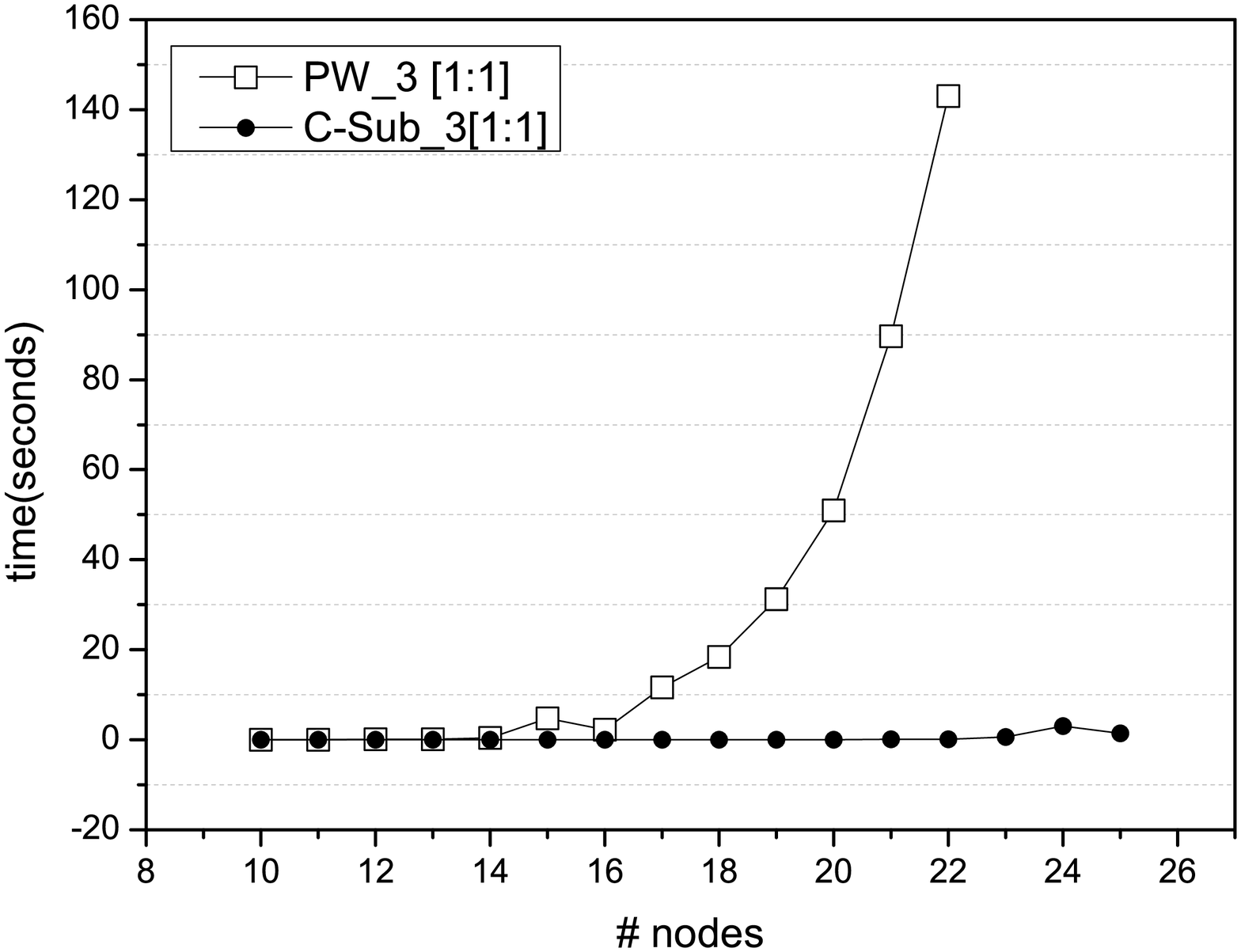}& \includegraphics [width=0.515\textwidth]{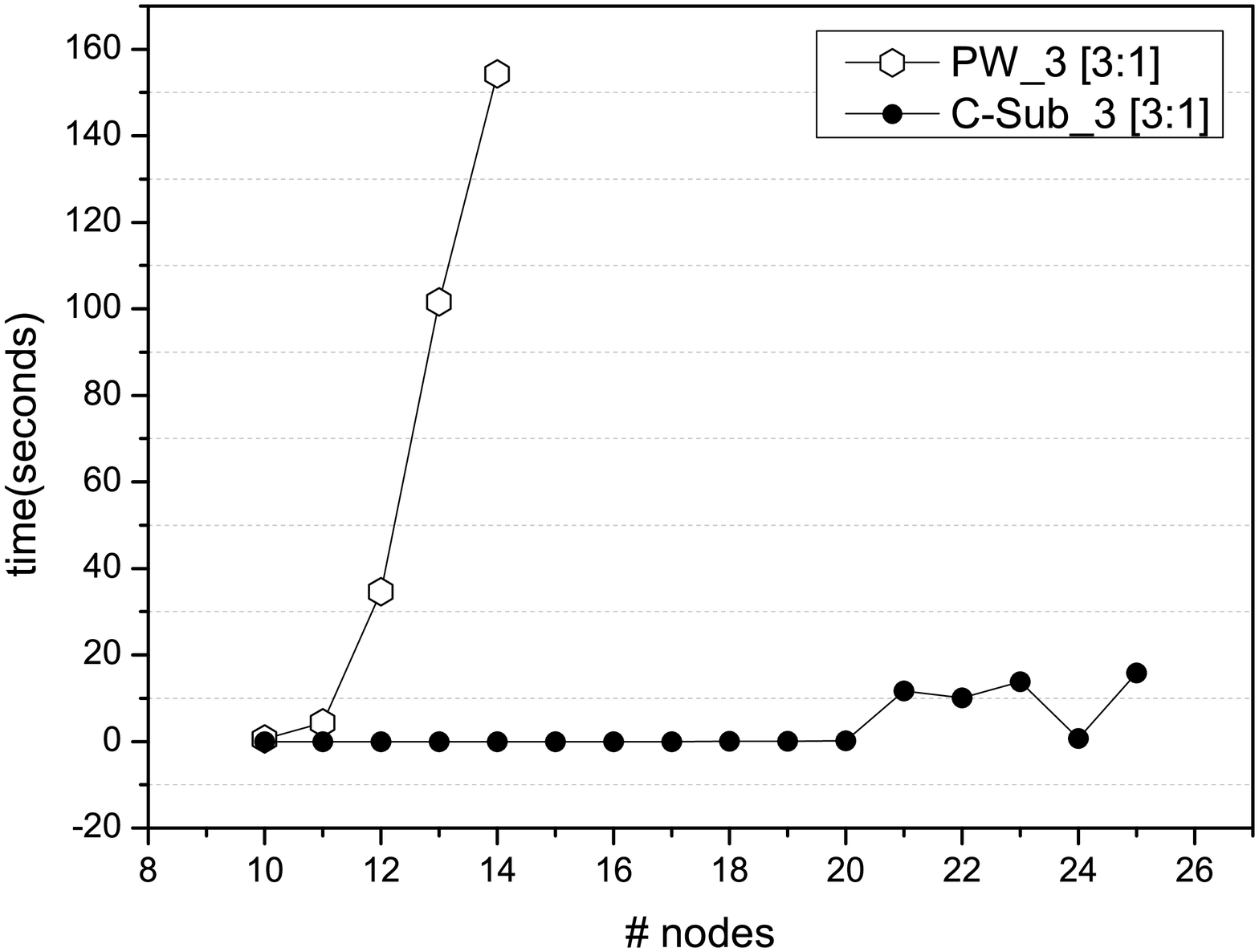} \\
\end{tabular}
\caption{Plots showing the execution time of the PW approach and the C-Sub approach}
\end{center}
\end{figure}

The second experiment is to further study how the increase of density of PrAGs affects the computation time of the two approaches. 

\begin{table}[!h]
\begin{center}
\begin{tabular}{|l|c|c|c|c|c|c|}
  \hline
$i:1$ &  1:1  &2:1 & 3:1 & 4:1 & 5:1 & 6:1\\
 %   & \scriptsize (secs)  &\scriptsize & \scriptsize (secs) &\scriptsize  \\
  \hline
  \hline
\parbox{3.5cm}{PW\_3 [$i:1$]  (secs) \\ \#nodes = 10} & 0.026& 0.141 & 0.694 & 1.641 & 2.781 & 4.056   \\
\hline
\parbox{3.7cm}{C-Sub\_3[$i:1$]  (secs) \\   \#nodes = 20}& {0.255}& {0.367} & {0.090 } & 0.014 &0.003 & 0.001 \\
  \hline
\end{tabular}
\end{center}
\caption{Average execution time of the PW approach and the C-Sub approach w.r.t. the changing of density of edges}
\end{table}

\begin{figure}[!ht]
\begin{center}
\begin{tabular}{l}
\vspace{0cm}\includegraphics [width=0.65\textwidth]{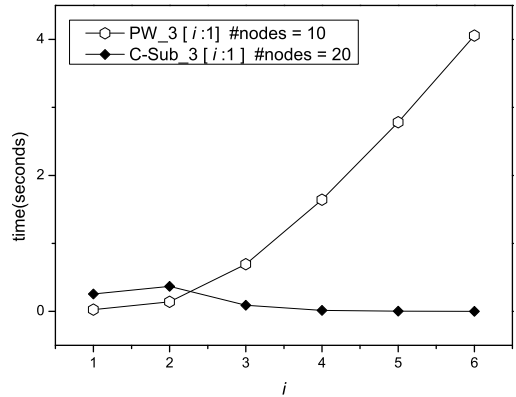}
\end{tabular}
\caption{Plot showing the average execution time of the PW approach and the C-Sub approach w.r.t. the changing of density of edges}
\end{center}
\end{figure}

As shown in Table 5 and Figure 7, the configuration for the PW approach is: the size of extension is 3, the number of nodes is 10, and the density of edges ranges from $1:1$ to $6:1$.  The configuration of the C-Sub approach is similar to that of the PW approach, the only difference is that the number of nodes is 20 for the C-Sub approach (in that when the number of nodes is less than 10, the average computation time of the PW approach is close to 0). The average execution time of the PW approach increases sharply with the increase of the density of edges, while the execution time of the C-Sub approach decreases with the increase of the density of edges.

The third experiment is study how the average execution time of the PW approach and the C-Sub approach changes with respect to the changing of the size of the extension.  In this experiment, the configuration for the two approaches is:  the size of extension is 3 and 5, the number of nodes  ranges from 10 to 25, and the density of edges is $2:1$.

\begin{table}[!hp]
\begin{center}
\begin{tabular}{|r|c|c||c|c|}
  \hline
  \#nodes & \footnotesize PW\_3 [2:1] &\footnotesize PW\_5 [2:1]& \footnotesize C-Sub\_3 [2:1]  &\footnotesize C-Sub\_5 [2:1]  \\
    & \scriptsize (secs/timeout)  &\scriptsize (secs/timeout) & \scriptsize (secs/timeout) &\scriptsize(secs/timeout)  \\
  \hline
  \hline
\footnotesize 10 & \footnotesize  0.120/0& \footnotesize  0.018/0 & \footnotesize 0.000/0 & \footnotesize 0.000/0     \\

  \footnotesize 11& \footnotesize 0.648/0& \footnotesize  0.042/0 & \footnotesize 0.000/0   & \footnotesize 0.000/0  \\

  \footnotesize 12& \footnotesize 3.141/0& \footnotesize  0.125/0  & \footnotesize 0.001/0  & \footnotesize 0.000/0  \\
  \footnotesize 13& \footnotesize 6.732/0& \footnotesize  0.368/0  & \footnotesize 0.005/0   & \footnotesize 0.000/0  \\

  \footnotesize 14& \footnotesize \underline{23.879}/1& \footnotesize   4.178/0 & \footnotesize 0.000/0  & \footnotesize 0.000/0   \\

  \footnotesize 15  & \footnotesize \underline{87.185}/8& \footnotesize   3.723/0   & \footnotesize 0.002/0   & \footnotesize 0.000/0   \\
  \footnotesize 16& \footnotesize \underline{112.569}/11& \footnotesize   \underline{43.981}/2 &\footnotesize 0.003/0 & \footnotesize  0.002/0  \\
  \footnotesize 17& \footnotesize \underline{107.201}/9& \footnotesize     \underline{70.016}/5 & \footnotesize 0.003/0& \footnotesize  0.000/0 \\
  \footnotesize 18& \footnotesize \underline{149.583}/16& \footnotesize   \underline{93.756}/8 & \footnotesize 0.003/0& \footnotesize 0.000/0 \\
  \footnotesize 19& \footnotesize \underline{159.580}/17& \footnotesize   \underline{108.857} /10&  \footnotesize 0.011/0& \footnotesize 0.003/0 \\
  \footnotesize 20&/20 & \footnotesize  \underline{151.422}/15 & \footnotesize 0.477/0& \footnotesize  0.000/0  \\
 \footnotesize 21&/20 & \footnotesize  \underline{155.704}/16 & \footnotesize {0.067}/0& \footnotesize  0.000/0  \\
  \footnotesize 22&/20 & \footnotesize \underline{171.270}/16   & \footnotesize \underline{14.925}/1& \footnotesize  0.003/0 \\
  \footnotesize 23&/20 & /20  & \footnotesize  {6.091}/0& \footnotesize  0.001/0 \\
  \footnotesize 24&/20 & /20  & \footnotesize \underline{12.434}/1& \footnotesize  0.222/0 \\

  \footnotesize 25&/20 & /20  & \footnotesize \underline{30.222}/3& \footnotesize  0.008/0 \\
% \footnotesize 26&/20 & /20  & \footnotesize \underline{ }/3& \footnotesize  2.857/0 \\
% \footnotesize 27&/20 & /20  & \footnotesize \underline{ }/3& \footnotesize  1.852/0 \\
% \footnotesize 28&/20 & /20  & \footnotesize \underline{ }/3& \footnotesize  1.852/0 \\
% \footnotesize 29&/20 & /20  & \footnotesize \underline{ }/3& \footnotesize  1.852/0 \\
% \footnotesize 30&/20 & /20  & \footnotesize \underline{ }/3& \footnotesize  1.852/0 \\
 \hline
\end{tabular}
\end{center}
\caption{The average execution time of the PW approach and the C-Sub approach with respect to different sizes of the extension}
\end{table}

\begin{figure}[!hp]
\begin{center}
\begin{tabular}{l}
\vspace{0cm}\includegraphics [width=0.65\textwidth]{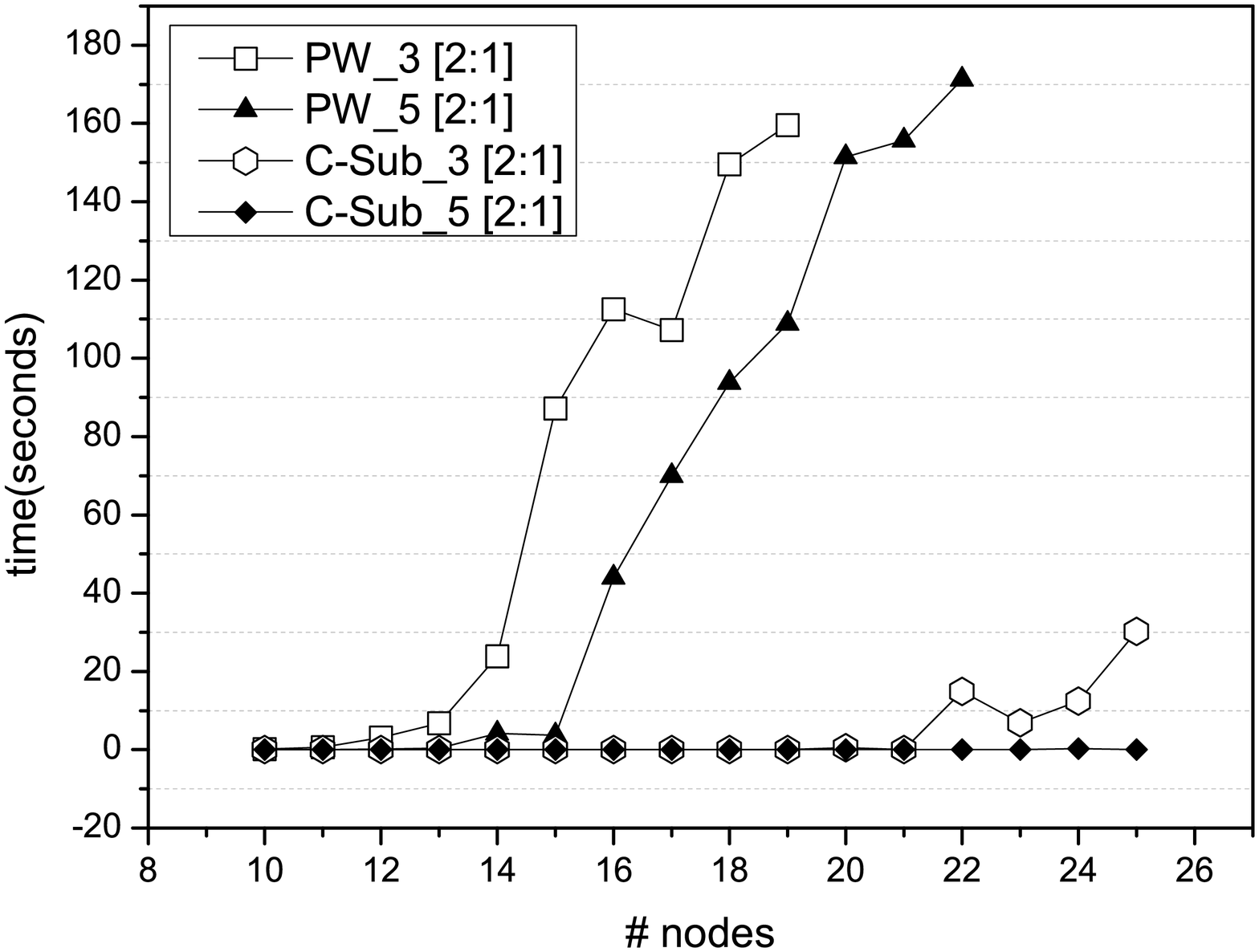}
\end{tabular}
\caption{Plots showing the average execution time of the PW approach and the C-Sub approach with respect to different sizes of the extension}
\end{center}
\end{figure}

According to the results shown Figure 8 (corresponding to the data in Table 6),  the shapes of the graphs PW\_3 [2:1] and PW\_5 [2:1] are almost the same, which means that the average execution time of the PW approach does not fundamentally decrease  with the changing of the size of the extension. On the contrary, the average execution time of the C-Sub approach decreases to a great extent. The basic reason behind this phenomenon is that: since the complexity of the C-Sub approach is manly determined by the size of $I = A\setminus (E\cup E_G^+\cup E_G^-)$, with the increase of the size the extension $E$, the size of $I$ become smaller.

\section{Computational properties}
Based on the theory and the experimental results introduced in Sections 4 and 5, in this section, we briefly analyze some computational properties of our C-Sub approach (or briefly ``our approach'').
 
On the one hand, according to \textit{classical complexity theory}, by using the C-Sub approach, it holds that computing $p(E^{ad})$ and $p(E^{st})$ is polynomial time tractable, while under complete, preferred and grounded semantics, problems of determining $p(E^{co})$, $p(E^{pr})$ and $p(E^{gr})$ are still intractable. This is because: under complete and preferred semantics, we need to consider $|2^I|$ cases, while under grounded semantics, $|2^I|+|2^{E_G^+\cap E_G^-}|$ cases. However, theoretically,  {the C-Sub approach is more efficient,} in that:

%\paragraph{Computational efficiency in general}
%According to Propositions  \ref{th-ad-sub} and \ref{th-st-sub}, it holds that computing $p(E^{ad})$ and $p(E^{st})$ is polynomial time tractable, similar to the results presented in 
 %\cite{Fazzinga:ijcai13}. However, according to Propositions 
% \ref{th-co-sub},
 %\ref{th-pr-sub} and
%\ref{th-gr-sub},
%computing $p(E^{co})$, $p(E^{pr})$ and $p(E^{gr})$ is still intractable, in that: under complete and preferred semantics, we need to consider $|2^I|$ cases, while under grounded semantics, $|2^I|+|2^{E_G^+\cap E_G^-}|$ cases. {So,  from the perspective of classical complexity theory, while it might be impossible to change the worst cases complexity of computing $p(E^{co})$, $p(E^{pr})$ and $p(E^{gr})$. However, our new approach is more efficient, in that }
\begin{enumerate} 
\item most subgraphs are not necessary to be constructed and computed, i.e., the maximal number of subgraphs decreases from $|2^A|$ to $|2^I|$ (or $|2^I|+|2^{E_G^+\cap E_G^-}|$), where $I = A\setminus (E\cup E_G^+\cup E_G^-)$ is the set of remaining arguments; and
\item the size of the maximal subgraph decreases from $|A|$ to $|I|$ (or $|E_G^+\cap E_G^-|$).
\end{enumerate}

The efficiency of the C-Sub approach is evidenced by the empirical results. This approach not only dramatically decreases the time for computing $p(E^\sigma)$, but also has an attractive property, which is contrary to that of existing approaches:  the denser the edges of a PrAG are or the bigger the size of a given extension $E$ is, the more efficient our approach computes $p(E^\sigma)$.

% First, with the increase of the number of nodes, the computation time of the PW approach increases dramatically, while that of the C-Sub approach might not increase (as illustrated in Figure 6). Second, with the increase of the density of edges, the average execution time of the PW approach increases sharply, while that of the C-Sub approach decreases (as illustrated in Figure 7). Third, when the size of the extension becomes bigger, the  average execution time of the PW approach does not fundamentally decrease, while that of the C-Sub approach decreases to a great extent. 

%\paragraph{Fixed parameter tractability of new algorithms}
On the other hand, under complete and preferred semantics, since the complexity of the C-Sub approach is {mainly determined} by the size of the remaining arguments, which is usually much smaller than that of the whole set of arguments in a PrAG,   according to \textit{parameterized complexity theory},  {the problems of determining $p(E^{co})$ and $p(E^{pr})$ in the C-Sub approach are fixed-parameter tractable with respect to the size of remaining arguments. }Details are as follows.

In terms of parameterized complexity theory, \textit{the complexity of a problem is not only measured in terms of the input size, but also in terms of a parameter. The theory's focus is on situations where the parameter can be assumed to be small} \cite{Flum}. Let $Q$ be a classical problem, and $k$ be a parameter of the problem. A parameterized problem is denoted as $(Q, k)$. When binding $k$ to a fixed constant, in many cases, an intractable problem $Q$ can be made tractable. This  property is called \textit{fixed-parameter tractability} (FPT). More specifically, let  $n$ be the input size of a problem, and $f$ be a computable function that depends on a parameter $k$ of the problem. The complexity class FTP consists of problems that can be computed in $f(k)\cdot n^{\mathcal{O}(1)}$. 

In the setting of this paper, let $k = |I|$. Typically, $k$ is much smaller than the size of the PrAG (i.e., $|A|$). Under complete and preferred semantics, the complexity of determining the probability  that a set of arguments is an extension is dominated by the size of the set of remaining arguments (i.e., $|I|$). Formally, we have the following proposition.

%The denser the edges of a PrAG are or the bigger the size of a given extension $E$ is, the more efficient our approach computes $p(E^\sigma)$. 

%The inputs of our algorithms consist of two parts, which typically have vastly different sizes. The size of all arguments in $A$ can be very large, whereas the remaining set of arguments $I$ tend to be fairly small. As a parameter, we choose the size of $I$. 

\begin{proposition} \label{pro-ftp}
Let $G^p = (A,R, p)$ be a PrAG, and $E\subseteq A$ be a conflict-free set of arguments. Let $k = |I|$ where $I$ is the set of remaining arguments of $G$ with respect to $E$. Let $Pr_G(E^{co})$ and $ Pr_G(E^{pr})$ be the problems of determining the probability  $p(E^{co})$ and $p(E^{pr})$ respectively. It holds that $(Pr_G(E^{co}), k)$ and $(Pr_G(E^{pr}), k)$  belong to FTP. 
\end{proposition}

\begin{proof}
First, under preferred semantics, the algorithm (Alg. 2) consists of the following two parts. The first part (Lines 2 - 10; Lines 18 - 33) is the difficult core of the algorithm. In this part,  there are $2^k$ calls and in each call, the procedure $\mathit{verify\_nonempty\_adm}(\mathcal{L})$ may be intractable. However, since the size of the subgraphs induced by $B^\prime$ is less than $k$, the time for executing $\mathit{verify\_nonempty\_adm}(\mathcal{L})$ is dependent on $k$, denoted as $g(k)$. The second part (Line 1, Lines 11 - 17) is tractable. The execution time of this part can be bounded by $n^{\mathcal{O}(1)}$ where $n=|A|$. So, the overall execution time can be bounded by $2^k\cdot g(k) + n^{\mathcal{O}(1)}= f(k)+ n^{\mathcal{O}(1)}\leq f(k)\cdot n^{\mathcal{O}(1)}$ where $f(k) = 2^k\cdot g(k) $.  Hence, $(Pr_G(E^{pr}), k)$ belongs to FTP. 

Second, in terms of Propositions \ref{th-co-sub} and \ref{th-pr-sub}, the algorithm under complete semantic (not presented in the present paper) is similar to the one under preferred semantics. The difference is that under complete semantics, $\mathit{verify\_nonempty\_adm}(\mathcal{L})$ is not executed. So, it holds that $(Pr_G(E^{co}), k)$ belongs to FTP. 
\end{proof}

Note that under grounded semantics,  since usually it may be not the case that $\max\{|E^+_G\cap E^-_G|, |I|\}$ is much smaller than $|A|$,  Proposition \ref{pro-ftp} can not be applied to grounded semantics. 

\section{Related work}
%
%\subsection{Related work}
In this paper, we have proposed a new approach (the C-Sub approach) to formulate semantics of probabilistic argumentation, and analyzed its computational properties on the basis of an empirical study. To the best of our knowledge, our approach is the first attempt to systematically study how to compute the semantics of probabilistic argumentation without (or with less) construction and computation of subgraphs not only under admissible and stable semantics, but also under other semantics including complete, grounded and preferred. In this section, we give a discussion about some related work.

\paragraph{Independence assumption of arguments} 
In this paper, we assume the independence of arguments appearing in a graph. A theoretical foundation for this assumption is originally formulated by Anthony Hunter in a series of his work \cite{Hunter:comma, Hunter:IJAR, Hunter:IJAR2013} from the justification perspective on the probability of an argument:%For a probabilistic argument graph in which only arguments are assigned with probabilities, the basic idea of  is as follows.

\begin{quote}
\textit{For an argument $\alpha$ in a graph $G$, with a probability assignment $p$, $p(\alpha)$ is treated as the probability that $\alpha$ is a justified point (i.e. each is a self-contained, internally valid, contribution) and therefore should appear in the graph, and $1-p(\alpha)$ is the probability that $\alpha$ is not a justified point and so should not appear in the graph. This means the probabilities of the arguments being justified are independent (i.e., knowing that one argument is a justified point does not affect the probability that another is a justified point). }
\end{quote}

The justification perspective can be further illustrated by the following example that was originally presented in \cite{Hunter:comma}.  Given two arguments $\alpha_1 = (\{p\}, p)$ and $\alpha_2 = (\{\neg p\}, \neg p)$ constructing from a knowledge base containing just two formulae $\{p, \neg p\}$, $\alpha_1$ attacks $\alpha_2$ and vice versa. In terms of classical logic, it is not possible that both arguments are true, but each of them is a justified point. So even though logically $\alpha_1$ and $\alpha_2$ are not independent (in the sense that if one is known to be true, then the other is known to be false), they are independent as justified points.

In existing literature, some models and algorithms depend on an independence assumption of  arguments and/or attacks \cite{Li:TAFA, Fazzinga:ijcai13,  Hunter:comma, Hunter:IJAR, Fazzinga:sum13}, while others do not \cite{Thimm2012, Rienstra:AT, Hunter:IJAR2013, Davide2013}. For the former, probabilities are assigned to arguments and/or attacks, and the probability distribution over subgraphs can be generated based on the independence assumption. For the latter, users directly specify the unique probability distribution over the set of subgraphs.  %Since the probability is not assigned to individual arguments, the independence assumption is not necessary. 
{There are pros and cons about whether the independence assumption is used or not. 
On the one hand,} by using the independence assumption, it can be more efficient to use the probability assignment to arguments and/or attacks and then generate the probability distribution over subgraphs. But, as discussed in \cite{Hunter:IJAR2013}, whilst the independence assumption is useful in some situations, it is not always appropriate. {On the other hand, when the independence assumption is avoided, the dependence relation between arguments can be properly represented. However, in this way, }  users have to specify probability distribution over the set of subgraphs, whose number is exponential with respect to the arguments and/or attacks. And, in many cases users may not be aware of the probability value that should be assigned to a possible world (subgraph) which may represent a complex scenario  \cite{Fazzinga:TCL}.  In this sense, in many situations, the models without independence assumption might not be applicable.

So, with regard to whether an independence assumption is used or not, there are both advantages and disadvantages.  %How to balance the problem of efficiency and that of expressivity is sti. % there are different models for combining argumentation and probability theory. Among them, some assign probabilities to arguments \cite{Dung:comma,Rienstra:AT, Hunter:comma}, while others assign probabilities to attacks \cite{Hunter:IJAR2013} or both arguments and attacks \cite{Li:TAFA}. Meanwhile, some rely on an independence assumption of arguments \cite{Li:TAFA, Fazzinga:ijcai13,  Hunter:comma, Hunter:IJAR, Fazzinga:sum13}, while others do not \cite{Thimm2012, Rienstra:AT, Hunter:IJAR2013}. 

\paragraph{Complexity analysis and algorithms for probabilistics argumentation}
Computational issues of probabilistic argumentation have been deeply investigated in recent years. 

On the one hand, Fazzinga et al studied  the complexity problem of determining the probability that a set of arguments is an extension under a given semantics \cite{Fazzinga:ijcai13}. The results show that under admissible and stable semantics,  the problem belongs to $\mathit{PTIME}$, while under complete, grounded, preferred and ideal semantics, the problem is $\mathit{FP}^{\sharp P}$. However,  the existing work only studied the complexity problems from the perspective of classical complexity theory. {The corresponding problems from the perspective of parameterized complexity theory have not been explored. }  

On the other hand, since using a brute-force algorithm to evaluate the probability of a set of arguments being an extension is computationally prohibitive, in existing work, an approximate approach (called the Monte-Carlo simulation approach) has been proposed to cope with this problem \cite{Li:TAFA}, which was significantly improved in \cite{Fazzinga:sum13}  by reducing the sample space of computation. Corresponding to these approximate approaches, however, little attention has been paid to the development of exact approaches. Our approach presented in this paper is a first step in this direction.

\paragraph{Efficient algorithms based on the structural properties of graphs}
{Since an abstract argumentation framework (argument graph) can be viewed as a digraph, applying various properties of existing graph theory to argumentation is not new. For instance, when an argument graph satisfies some properties  (acyclic, symmetric, bipartite,  etc), there exist tractable algorithms to compute its semantics \cite{Dunne:AIJ2007,Sylvie:ECSQARU}; when an argument graph has bounded tree-width, there exist fixed-parameter algorithms \cite{Wolfgang:AIJ2012B}; when decomposing an argument graph based on the notion of strongly connected components, the efficiency of computation can be significantly improved \cite{Liao:amai}; {by mapping  the notion of a kernel, a semikernel and a maximal semikernel in a directed graph  \cite{Galeana} respectively to the notion of a stable set, an admissible set and a preferred extension in an argumentation framework  \cite{Coste-Marquis2005, Dyrkolbotn2014}, in terms of \cite{Dimopoulos1996}, we may infer that the complexity results and algorithms related to kernels and semikernels can be applied to formal argumentation}. Beside the structural properties that have been applied to argumentation, {notions of kernels and semikernels have also been connected to logic programs and default theories.} According to \cite{Dimopoulos1996}, every normal logic program can be transformed into a graph.  The stable, partial stable and well-founded semantics correspond to  kernels, semikernels and the initial acyclic part, respectively. Meanwhile, according to \cite{Walicki}, it is an equivalence relation between the problem of the existence of kernels in digraphs and satisfiability of propositional theories (SAT). Thanks to this relation, algorithms for computing kernels can be applied to computing the semantics of logic programs. 

It is worth to note that although structural properties of digraphs have been exploited in Dung's abstract argumentation \cite{Dung:AIJ}, we have not found solutions to apply these properties to the efficient computation of the semantics of probabilistic argumentation. {So, in this paper, based on some basic structural properties of digraphs and formal argumentation \cite{Galeana, Dung:AIJ}, we have defined properties that can be used to characterize subgraphs of a PrAG. These properties are established in the setting of probabilistic argumentation where the appearance of arguments is related to a given extension $E$, and on the basis of the original definition of extensions under different argumentation semantics \cite{Dung:AIJ}. Despite of their simplicity, these properties lay a concrete foundation to define a new methodology to formulate and compute the semantics of probabilistic argumentation.}}

%By comparison, this paper treats with probablistic argument graphs, and has little relation to the above-mentioned structural properties from a graph-theoretic perspective. Instead, the properties for characterizing subgraphs have a closer relation to the original definition of extensions under different argumentation semantics in \cite{Dung:AIJ}. For example, given a conflict-free set $E$ of arguments, $E$ is admissible if and only if ``each argument in $E$ is acceptable with respect to $E$''.  This condition corresponds to Prop2: ``all arguments in $E_G^-\setminus E_G^+$ do not appear in the subgraph''. Although these properties are rather obvious, to the best of our knowledge, they have not 

%how to exploit some specific properties of digraphs in formulating the semantics of probabilistic argument graphs is still a new problem.  In this paper, structural properties are used to characterize the subgraphs that have a given extension.  They are defined on the basis of the original definition of extensions under different argumentation semantics in \cite{Dung:AIJ}. For example, given a conflict-free set $E$ of arguments, $E$ is admissible if and only if ``each argument in $E$ is acceptable with respect to $E$''.  This condition corresponds to Prop2: ``all arguments in $E_G^-\setminus E_G^+$ do not appear in the subgraph''. So, from the perspective of graph theory, these properties
\paragraph{Kernelization and parameterized algorithms}
The approach and results presented in this paper have a close relation to some existing work on kernelization and parameterized algorithms, which have been extensively studied in the past two decades.  Kernelization is a systematic approach to study polynomial-time preprocessing algorithms, such that the ``easy parts'' of a problem instance  can be solved efficiently, and the problem instance is reduced to its computationally difficult ``core'' structure (the problem kernel of the instance) \cite{Lokshtanov}. If the size of the kernel can be effectively bounded in terms of a fixed-parameter alone, then the problem is \textit{fixed-parameter tractable}  (FPT) \cite{Flum}. In recent years, fixed-parameter algorithms were developed in the setting of Dung's abstract argumentation, by exploiting some important parameters for graph problems, such as the \textit{tree-width} \cite{Dunne:AIJ2007, Wolfgang:AIJ2012B} and the \textit{clique-width} \cite{Wolfgang:COMMA, Wolfgang:AIJ2012} of a graph. 

The C-Sub approach presented in this paper can be understood as a kind of kernelization. The novelty of this approach lies in the fact that new properties are defined to characterize the subgraphs of a PrAG with respect to a given extension.

%n \cite{Fazzinga:sum13,Fazzinga:ijcai13}, it has been made clear that for some (i.e., admissible and stable), the two problems are both tractable.  As far as the admissible and stable semantics are concerned, their results show that the exact value of $Pr_G^{sem}(S)$ can be determined in polynomial time, without enumerating the possible worlds. For the intractable cases, existing work relies solely on the approximate approach (Monte-Carlo simulation approach).  In Fazzinga et al. [2013], they devised an optimized Monte-Carlo simulation approach that is able to estimate $Pr_G^{sem}(S)$, with $sem\in \{co, gr, pr\}$, using much fewer samples than the original approach proposed in  Li et al. [2011]. Specifically, the proposed approach exploits the tractability results of \cite{Fazzinga:sum13}. 

\section{Conclusions and future work}
Probabilistic argumentation is an emerging direction in the area of formal argumentation.  In this paper, we have studied the formulation and computation of semantics of probabilistic argumentation. The main contributions of this paper are two-fold. 

On the one hand, conceptually, we define specific properties to characterize the subgraphs of a PrAG with respect to a given extension, such that the probability of a set of arguments $E$ being an extension can be defined in terms of these properties, without (or with less) construction of subgraphs. {The theoretical results in this paper show that under admissible and stable semantics,} computing a set of arguments being an extension of a PrAG is polynomial time tractable; {under complete and preferred semantics, the problems of determining $p(E^{co})$ and $p(E^{pr})$ in our C-Sub approach are fixed-parameter tractable with respect to the size of remaining arguments.}

%obtained without constructing subgraphs, while under some other semantics (complete, preferred and grounded), only much fewer subgraphs are constructed and computed. 
On the other hand, computationally, we take preferred semantics as an example, and develop algorithms to evaluate the efficiency of our approach. The empirical results show that  our approach not only dramatically decreases the time for computing the semantics of probabilistic argumentation, but also has an attractive property, which is contrary to that of existing approaches:  the denser the edges of a PrAG are or the bigger the size of a set of arguments $E$ is, the more efficient our approach computes the probability of $E$ being an extension of the PrAG.

Future work is as follows. First, in this paper we deal with the probabilistic argument graphs (PrAGs) in which probabilities are assigned to arguments. However, when probabilities are assigned to attacks or to both arguments and attacks, the formalisms and algorithms corresponding to the ones in this paper are expected to be different. In \cite{Hunter:IJAR2013}, only attacks are assigned with probabilities. So, it would be interesting to combine the theory presented in \cite{Hunter:IJAR2013} with the approach presented in this paper. Meanwhile, one may consider to extend our approach to the cases where both arguments and attacks are assigned with probabilities, similar to the work presented in \cite{Fazzinga:sum13}. %So, it is worth to investigate the corresponding research problems oriented to PrAGs in which probabilities are assigned to attacks, or both arguments and attacks.  
 Second, an independence assumption of arguments is used in the paper. Although it is useful in some situations, but not always appropriate. So,  a further step is to develop the corresponding approaches without this assumption. Third, as mentioned above, under grounded semantics, we have not obtained a conclusion that our approach is fixed-parameter tractable. Further analysis about this issue is needed. Fourth, in the empirical study,  we have only considered preferred semantics. The algorithms and experiments under other semantics (especially grounded semantics) are also important.

\section*{Acknowledgment}
We are grateful to the reviewers of this paper for their constructive and insightful comments. The research reported in this paper was partially supported by the National Research Fund Luxembourg (FNR) and Zhejiang Provincial Natural Science Foundation of China (No. LY14F030014). 

\bibliographystyle{splncs}
%\bibliography{liao}

\end{document}